\documentclass[onefignum,onetabnum]{siamart220329}

\usepackage{graphicx,algorithmic}
\usepackage{amssymb}
\usepackage{booktabs}
\usepackage{multirow}
\usepackage{bm}
\usepackage{dsfont}
\usepackage{mathrsfs}
\usepackage{grffile}
\usepackage{url}
\usepackage{nicefrac}
\usepackage{enumitem}
\usepackage{pifont}
\usepackage{tikz}
\usetikzlibrary{arrows,arrows.meta,positioning}
\usepackage[caption=false]{subfig}
\usepackage{caption}

\newcommand{\TheTitle}{The Effective Number of Nonzeros}
\newcommand{\TheAuthors}{H. He, H. Wang, J. Wang, and H. Zeng}
\headers{\TheTitle}{\TheAuthors}

\title{The Effective Number of Nonzeros: Theory and Regularization for Sparse Recovery%
\thanks{Submitted to the journal's Numerical Algorithms for Scientific Computing section. This work was supported by the General Program of the National Natural Science Foundation of China (NSFC) under No. 12571326.}}

\author{Haoyu He\thanks{School of Information Science and Technology, ShanghaiTech University, Shanghai, China
(\email{hehy12024@shanghaitech.edu.cn}).}
\and Hao Wang\thanks{School of Information Science and Technology, ShanghaiTech University, Shanghai, China
(\email{haw309@gmail.com}). Corresponding author.}
\and Jiashan Wang\thanks{Department of Mathematics, University of Washington, Seattle, WA, USA
(\email{jsw1119@gmail.com}).}
\and Qiankun Shi\thanks{School of Computer Science and Engineering, Sun Yat-sen University
(\email{shiqk@mail2.sysu.edu.cn}).}}

\ifpdf
\DeclareGraphicsExtensions{.pdf,.png,.jpg}
\fi

\numberwithin{equation}{section}
\newsiamthm{Def}{Definition}
\newsiamthm{prop}{Proposition}
\newsiamthm{ASM}{Assumption}
\newsiamremark{remark}{Remark}

\def\bb{{\bm b}}

\def\be{{\bm e}}

\def\bh{{\bm h}}

\def\bp{{\bm p}}
\def\bq{{\bm q}}

\def\bu{{\bm u}}
\def\bv{{\bm v}}

\def\bx{{\bm x}}
\def\by{{\bm y}}

\newcommand{\bpi}{\bm{\pi}}
\newcommand{\Hcal}{\mathcal{H}}
\newcommand{\Supp}{\mathrm{Supp}}

\newcommand{\enz}{\mathrm{ENZ}}

\def\Rcal{{\mathcal R}}

\begin{document}
\maketitle

\begin{abstract}
Classical sparse recovery treats all nonzero entries equally, though numerical noise often creates long tails of negligible coefficients. This paper develops an entropy-based notion of effective sparsity to measure the coefficients carrying significant mass. The central quantity, the effective number of nonzeros (ENZ), is obtained by exponentiating the Shannon entropy of the normalized magnitude distribution. We show that ENZ decomposes exactly into the support cardinality multiplied by a distributional efficiency factor, thereby making precise its relation to the $\ell_0$ count and explaining how it discounts uninformative coefficients. Furthermore, the Shannon ENZ is embedded into a parallel Rényi family that recovers several scale-invariant sparsity measures, including the $\ell_1/\ell_2$ ratio, as special cases. We then prove a stability result under a restricted isometry condition, establishing an explicit bound that depends on the tail energy, measurement perturbation, and restricted isometry constant. For computation, a separable unnormalized entropy surrogate is introduced to avoid global coupling. Numerical experiments on sparse signal recovery and gradient-domain image denoising demonstrate that the resulting regularizer is robust, computationally efficient, and competitive with standard sparsity penalties.
\end{abstract}

\begin{keywords}
effective sparsity, effective number of nonzeros, normalized entropy, sparse recovery, restricted isometry property, inverse problems
\end{keywords}

\begin{MSCcodes}
65F22, 65K10, 90C26, 94A12
\end{MSCcodes}

\section{Introduction} \label{sec:intro}

Recovering sparse signals from underdetermined linear systems is a cornerstone problem in machine learning, high-dimensional statistics, and signal processing. Consider 
\begin{equation}\label{eq:linear_system}
\bb   = A \bx + \boldsymbol{\varepsilon},
\end{equation}
where $A  \in \mathbb{R}^{m \times n}$ is the sensing matrix with $m \ll n$, $\bb \in \mathbb{R}^m$ represents the observed data, and $\boldsymbol{\varepsilon} \in \mathbb{R}^m$ denotes additive noise. Under the classical paradigm, the true signal $\bx$ is assumed to be strictly sparse, and the objective is to reconstruct $\bx$ by minimizing the combinatorial $\ell_0$ cardinality count \cite{blumensath2008iterative, candes2006stable, chen2001atomic, donoho2006compressed, fan2001variable, tibshirani1996regression}:
\begin{equation}\label{eq:l0_minimization}
\min_{\bx}\ \|\bx\|_0 \quad \text{subject to}\quad  {A}\bx=\bb.
\end{equation}
Because problem \eqref{eq:l0_minimization} is known to be NP-hard and exhibits extreme sensitivity to noise \cite{natarajan1995sparse}, a vast literature has focused on developing computationally feasible surrogates. These methods range from the convex $\ell_1$ relaxation \cite{beck2009fast,tibshirani1996regression,chen2001atomic,candes2006robust,Candes2008Enhancing} to non-convex penalties such as $\ell_p$ ($0 < p < 1$) \cite{chartrand2008restricted}, $\ell_1 - \ell_2$ \cite{yin2015minimization}, and the scale-invariant $\ell_1/\ell_2$ ratio \cite{rahimi2019scale,xu2021analysis}. While these continuous surrogates have achieved notable empirical successes, their theoretical recovery guarantees almost universally rely on the Restricted Isometry Property (RIP), where stability thresholds remain strictly tied to the literal cardinality of the signal's support \cite{CRMATH,blanchard2011compressed}.

Despite significant progress, a fundamental discrepancy persists between literal $\ell_0$ sparsity and the physical nature of signals encountered in realistic, noisy environments. In practice, solutions to $ {A}\bx=\bb$ inevitably contain a ``long tail" of numerically small but non-zero components, stemming from precision limits, modeling inaccuracies, or ambient noise \cite{lopes2013estimating}. While these components contribute negligibly to the signal's energy or underlying structure, they indiscriminately inflate the discrete $\ell_0$ count. This tension is particularly acute in compressed sensing theory, where stable recovery typically guarantees success under a $2k$-RIP condition provided the vector contains at most $k$ non-zero entries \cite{candes2006stable,foucart2013mathematical}. When infinitesimal noise perturbations are treated as genuine support elements, the apparent sparsity level $k$ inflates dramatically, forcing the RIP to hold at artificially higher orders. This leads to a theoretical paradox: classical theoretical analysis imposes unnecessarily stringent measurement requirements primarily to safeguard irrelevant tail components, yielding overly pessimistic sampling bounds that fail to reflect the true geometric complexity of the meaningful signal. Consequently, rigid cardinality counting alone is an inappropriate notion of sparsity for noisy inverse problems.

The paradigm shift from literal counting to \emph{effective} quantification is not isolated to sparse recovery, but reflects a foundational principle across modern computational mathematics, learning theory, and statistical physics. In statistical learning foundations, for instance, the raw count of system parameters is widely known to be an unreliable proxy for generalization; instead, metrics such as the Vapnik-Chervonenkis (VC) dimension are utilized to capture the system's  {effective} capacity and true complexity \cite{vapnik1995nature}. 
In large-scale inverse problems, early termination strategies implicitly prevent overfitting to noise by actively focusing computational resources on variables corresponding to dominant,  {effective} eigenvalues, rather than wasting iterations on the dense, non-informative tails of the spectrum \cite{engl1996regularization}. Parallel formulations of this information-theoretic softening have long sustained profound lineages in other scientific disciplines, notably via Meucci's  {effective} number of assets in quantitative finance \cite{meucci2007managing} and and Hill numbers in ecological diversity \cite{hill1973diversity}. Synthesizing these insights underscores a common computational principle: when physical systems are plagued by ambient noise, continuous, distribution-based metrics offer a far more stable characterization of true system complexity than rigid binary counts \cite{roy2007effective}.

Motivated by this universal principle, this paper focuses on the effective number of nonzeros (ENZ) as a continuous regularization model for robust sparse recovery. Grounded in an information-theoretic framework \cite{shannon1948mathematical,cover1999elements,jaynes1957information}, the ENZ characterizes the intrinsic complexity of a signal by evaluating its continuous distributional concentration rather than the binary state of its support set. Under this continuous formulation, small-amplitude noise coefficients carry negligible weight in the underlying entropy calculation, meaning that the overall metric naturally suppresses the influence of long-tail perturbations without affecting the true signal profile. By taking ENZ directly as our variational core, we establish a mathematically stable proxy for cardinality that accurately reflects the effective dimensionality of the system while maintaining a smooth and computationally scalable optimization landscape.

\subsection*{Contribution}

In this paper, we study the effective number of nonzeros (ENZ) as a sparsity model for noisy sparse recovery. The main premise is that literal support size and informative support size are different; a recovered signal can be dense in the $\ell_0$ sense but still be governed by a few dominant coefficients. We develop the theory, extensions, and recovery formulations around this distinction. Our main contributions are threefold:

First, we introduce ENZ as a continuous, magnitude-aware measure of sparsity. We prove a mathematical decomposition showing that the Shannon ENZ factors the discrete support size into a product with a continuous distributional efficiency term. This identity explains how entropy discounts small, uninformative coefficients. We also extend this framework to a Rényi family, showing that scale-invariant measures like the $\ell_1/\ell_2$ ratio arise as special cases.

Second, we establish a stability result under a restricted isometry condition for approximately sparse signals. The resulting bound shows that the error on the dominant support depends on the measurement noise and the energy of the negligible tails, rather than the full literal support size.

Third, to avoid the computational issues caused by global coordinate coupling in standard entropy, we introduce a separable, unnormalized entropy surrogate. We then propose a practical double-loop algorithm that uses an iterative re-scaling scheme and a smoothed L-BFGS solver to efficiently compute the solution .

 \subsection{Outline}

The rest of the paper is organized as follows. \Cref{sec.concept} motivates effective sparsity through empirical analysis of image and text data. \Cref{sec:effective_nonzeros} develops the theory of ENZ, including its definition, decomposition, and extension to the Rényi hierarchy. \Cref{sec:stability} proves the restricted-isometry stability result. \Cref{sec:entropy-model} details the computational formulation, the unnormalized surrogates, and the double-loop algorithm. Finally, \Cref{sec:numerical} reports numerical experiments, and \Cref{sec:conclusion} concludes the paper.

\subsection{Notation and Preliminaries}
\label{subsec:notation}

We denote $\mathbb{R}_+ = [0, \infty)$ and $\mathbb{R}^n_+ = \{\bx \in \mathbb{R}^n : x_i \geq 0, \ \forall i\}$. 
The $(n-1)$-dimensional probability simplex is defined as
\[
\Delta^{n-1} = \left\{ \bpi = (\pi_1, \dots, \pi_n) \in \mathbb{R}^n_+ : \sum_{i=1}^n \pi_i = 1 \ \pi_i \ge 0 \right\}.
\]
The $\ell_p$-norm of $\bx = (x_1, \dots, x_n) \in \mathbb{R}^n$ is  defined as $\|\bx\|_p = (\sum_{i=1}^n |x_i|^p)^{1/p}$. 
The support of $\bx$ is $\Supp(\bx) := \{i : x_i \neq 0\}$, and its cardinality is denoted by $\|\bx\|_0 = |\Supp(\bx)|$.
For a probability distribution $\bpi \in \Delta^{n-1}$, we define:

\begin{itemize}
    \item \emph{Shannon entropy} \cite{shannon1948mathematical}: 
    $H(\bpi) = -\sum_{i=1}^n \pi_i \log_2 \pi_i$, with the convention $0\log_2 0 = 0$.
    
    \item \emph{Rényi entropy} (order $\alpha > 0$, $\alpha \neq 1$) \cite{renyi1961measures}: 
    $R_\alpha(\bpi) = \frac{1}{1-\alpha} \log_2 \left( \sum_{i=1}^n \pi_i^\alpha \right)$, note that as $\alpha \to 1$, $R_\alpha(\bpi) \to H(\bpi)$.

    \item \emph{Kullback--Leibler (KL) divergence} between $\bp, \bq \in \Delta^{n-1}$ \cite{kullback1951information}:
    $D_{\mathrm{KL}}(\bp \| \bq) = \sum_{i=1}^n p_i \log_2 \frac{p_i}{q_i}$,
    with the conventions $0\log_2(0/0) = 0$ and $p_i\log_2(p_i/0) = +\infty$ if $p_i > 0$.
\end{itemize}

 \section{Motivation and Effective Sparsity}\label{sec.concept}

Before establishing the formal mathematical framework of the effective number of nonzeros (ENZ), we first illustrate the fundamental distinction between strict support cardinality and effective support concentration. In realistic scientific computing and inverse problems, recovered vectors are rarely strictly sparse; rather, small but non-zero entries routinely emerge due to ambient noise, discretization artifacts, or numerical modeling errors. While these infinitesimal components inflate the literal support size and make the signal appear high-dimensional under the $\ell_0$ metric, the underlying meaningful structure remains governed by a compact set of dominant coordinates. This section provides systematic empirical evidence of this phenomenon in both transform-domain images and latent semantic textual data, demonstrating that strict cardinality counts are overly rigid for the regimes in which sparse recovery algorithms are practically deployed.

\subsection{Effective Sparsity in Natural Images}\label{sec.effective.sparsity.images}

We first analyze natural images from both a transform-domain and a spatial-structure perspective to characterize the decay profile of spatial variations.

\paragraph{Wavelet-domain coefficient decay}
To examine the spectral energy distribution in a standard transform domain, we evaluate a benchmark collection of standard natural test images (including \emph{cameraman}, \emph{peppers}, and related standard profiles). Each image is normalized by its maximum pixel value and decomposed using a Daubechies-4 (db4) wavelet basis with four levels of decomposition, performed independently across each color channel where applicable. The resulting wavelet coefficients are sorted in descending order of magnitude and normalized by their maximum absolute value. 

As illustrated in \Cref{fig.test.image}, the empirical decay curves reveal a highly consistent, rapid decay  pattern across the entire dataset. While the overwhelming majority of the coefficients are strictly non-zero, their magnitudes plunge exponentially. Only a minute fraction of the largest coefficients captures the dominant energy and semantic structure of the image, while the remaining dense coefficients form a long, negligible tail. This clear concentration verifies that natural images natively exhibit pronounced effective sparsity rather than strict $\ell_0$ sparsity.

\begin{figure}[h]
\centering
\includegraphics[scale=0.33]{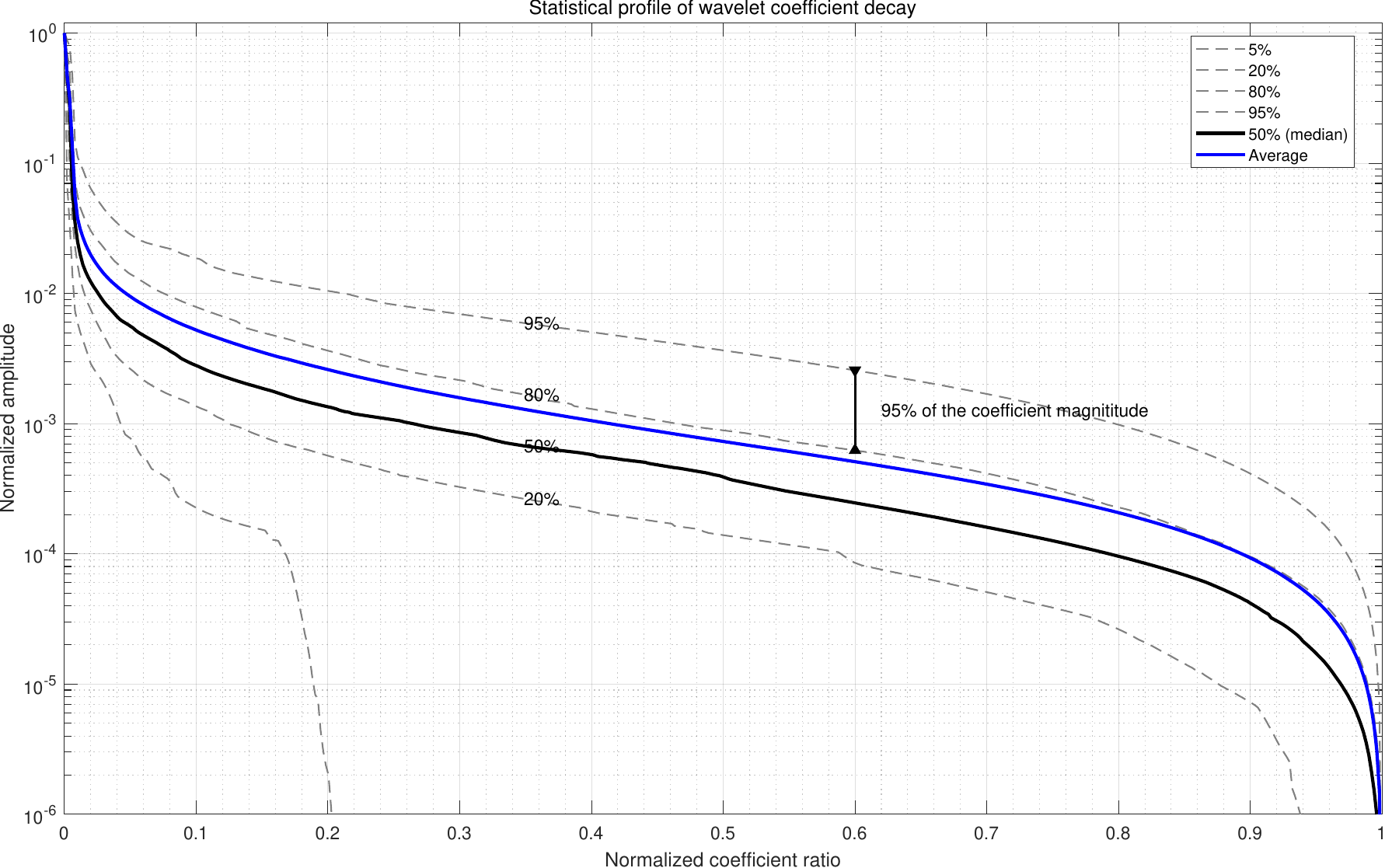}
\caption{Sorted magnitude decay curves of wavelet coefficients across a standard image benchmark. The thin lines represent individual image channels, while the solid blue and black lines denote the mean and median decay, respectively, bounded by dashed percentile envelopes. The rapid, heavy-tailed decay indicates that the primary image energy is concentrated on a small fraction of dominant coefficients, while the rest form an uninformative tail.} 
\label{fig.test.image} 
\end{figure} 

\paragraph{Directional total variation decay}
To investigate whether this structural concentration persists directly within the spatial domain, we analyze the finite differences of the image gradients. For an image $I \in \mathbb{R}^{m \times n}$, the forward finite difference operators are denoted by $(D_x I)_{i,j} = I_{i,j+1} - I_{i,j}$ and $(D_y I)_{i,j} = I_{i+1,j} - I_{i,j}$, yielding the discrete gradient vector $\nabla I_{i,j} = [D_x I_{i,j}, D_y I_{i,j}]^\top$. The anisotropic total variation (TV) is then defined as $\mathrm{TV}(I) = \|\nabla I\|_1 = \sum_{i,j} (|(D_x I)_{i,j}| + |(D_y I)_{i,j}|)$. 

To evaluate gradient sparsity, we collect the complete sets of horizontal and vertical difference magnitudes, sort them in descending order, and normalize them by their maximum values. As displayed in \Cref{fig.test.tv}, the horizontal and vertical TV components exhibit nearly identical, rapid power-law decay profiles. This spatial attenuation confirms that image gradients possess strong effective sparsity: the essential geometric transitions and edges are captured by a small subset of sharp boundaries, while the smooth regions generate a dense sea of non-zero but numerically negligible variations.

\begin{figure}[tbp]
    \centering 
    \subfloat[Vertical TV]{\includegraphics[width=0.48\textwidth]{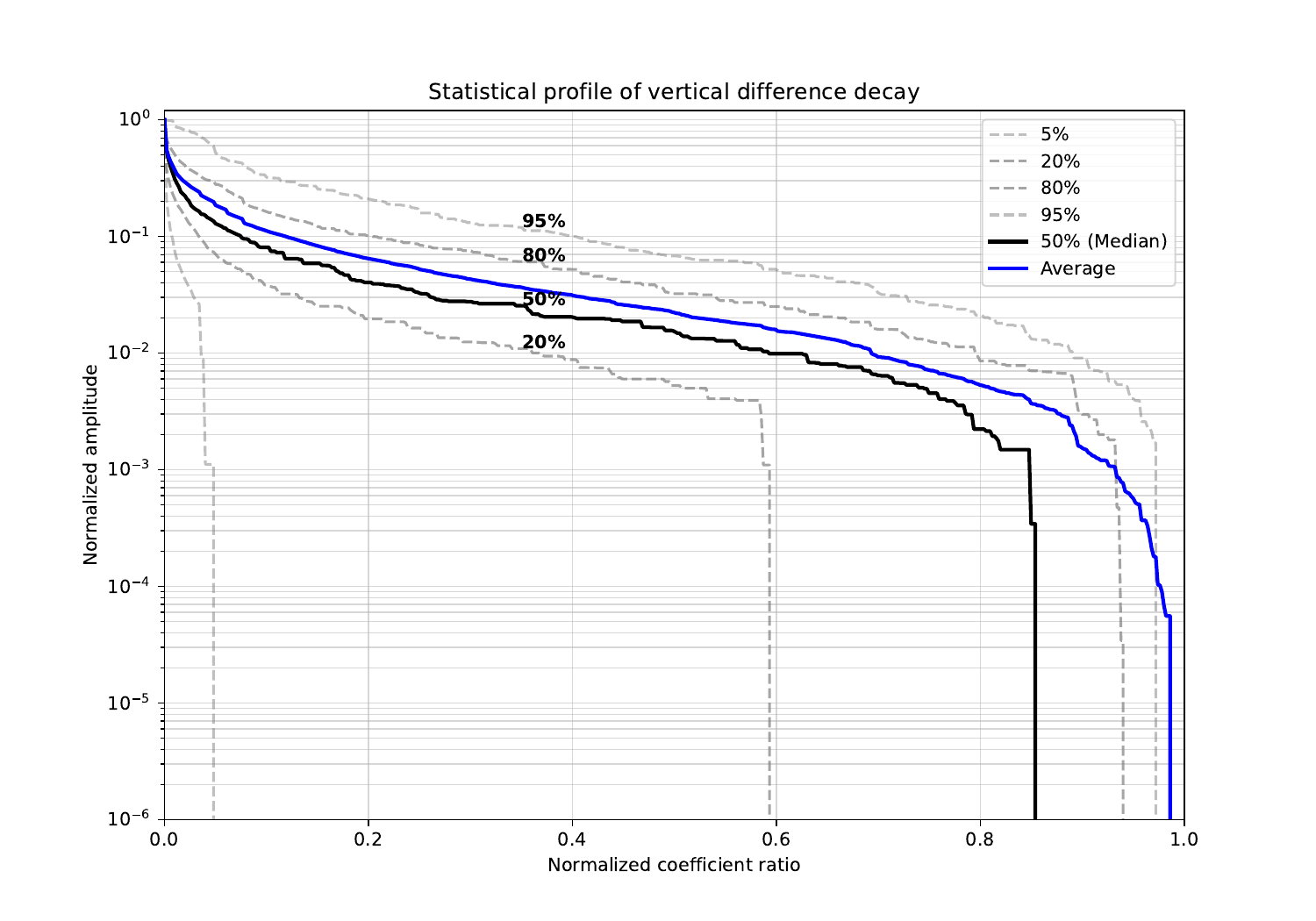}}\hfill
    \subfloat[Horizontal TV]{\includegraphics[width=0.48\textwidth]{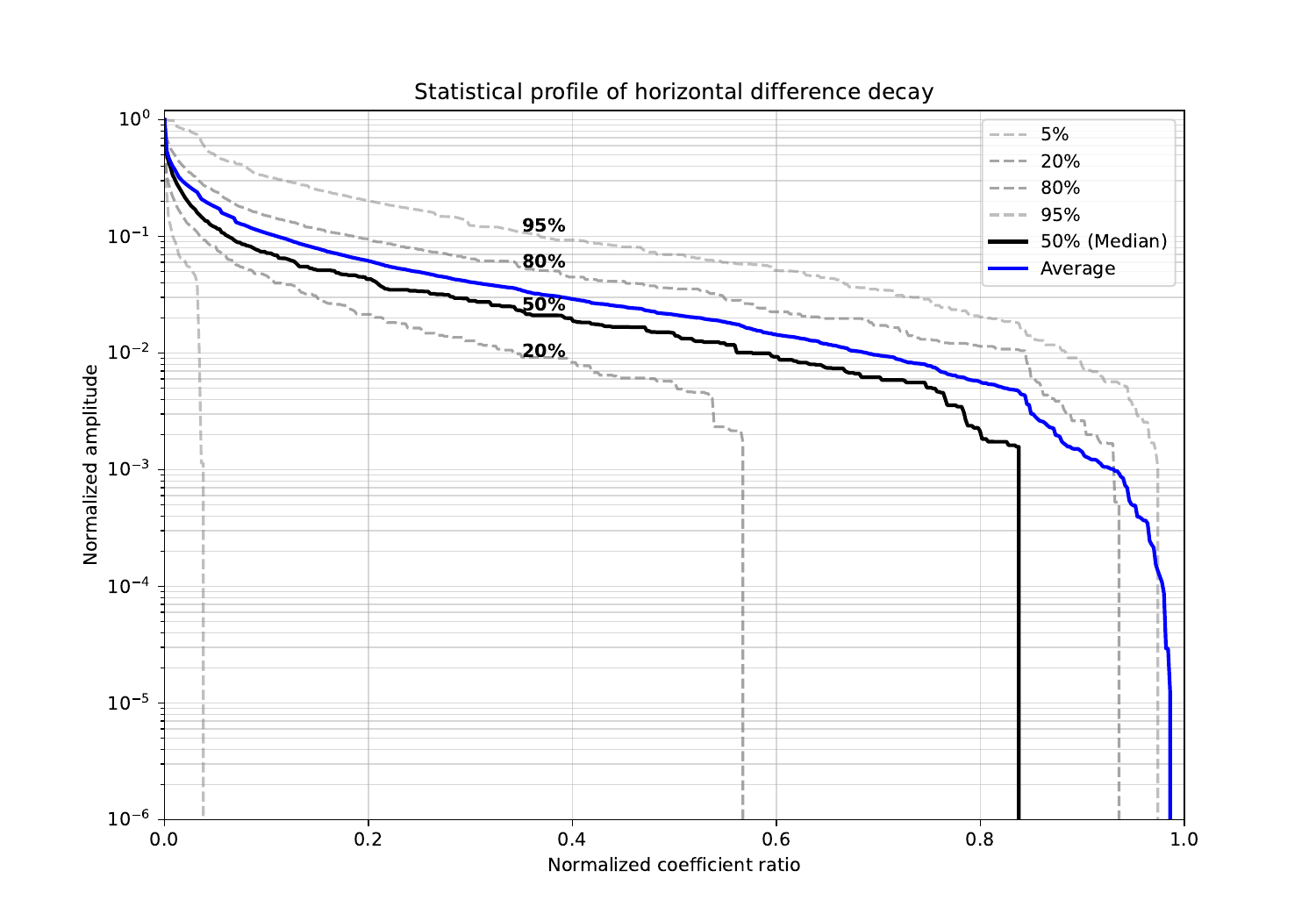}}
	\caption{Sorted magnitude decay of directional total variation (TV) coefficients. The profiles demonstrate nearly symmetrical and rapid spatial gradient attenuation across both horizontal and vertical directions, showing that natural image gradients are governed by a compact set of sharp transitions buried in dense, low-amplitude ambient variations.}
	\label{fig.test.tv}
\end{figure} 

\subsection{Spectral Decay and Effective Sparsity in Textual Data}

Unlike natural images, where structural concentration is typically verified in fixed functional bases, high-dimensional textual data presents effective sparsity within its latent semantic spaces. Following the spectral complex systems framework, we analyze the term-document matrices constructed from the 20 Newsgroups dataset \cite{20Newsgroups}, building on the premise that textual information constituting a single coherent topic is inherently low-rank \cite{thibeault2024low}. 

To evaluate this without global corpus distortion, we partition the documents into independent semantic categories and construct a local term-document matrix $X \in \mathbb{R}^{n \times d}$ for each topic using TF-IDF features (including bigrams) without row-wise $\ell_2$ normalization to preserve true magnitude variations. We then compute the singular value decomposition (SVD) of each matrix and quantify the informational variance via the normalized spectral energy $\lambda_k = \sigma_k^2 / \sigma_{\max}^2$, which corresponds to the normalized eigenvalues of the empirical covariance matrix.

The aggregated spectral profiles across all categories are shown in \Cref{fig.text.spectrum}. Although the ambient vocabulary space is exceptionally large ($d \approx 50,000$), the singular values decay precipitously according to a heavy-tailed power-law distribution. This demonstrates that the vast majority of semantic variance within a given topic is contained in a low-dimensional latent subspace. The dense, extended tail of non-zero singular values corresponds to idiosyncratic word usage or ambient noise that carries minimal predictive power, confirming that effective sparsity fundamentally dictates the intrinsic complexity of linguistic systems.

\begin{figure}[h]
	\centering
	\includegraphics[width=0.6\textwidth]{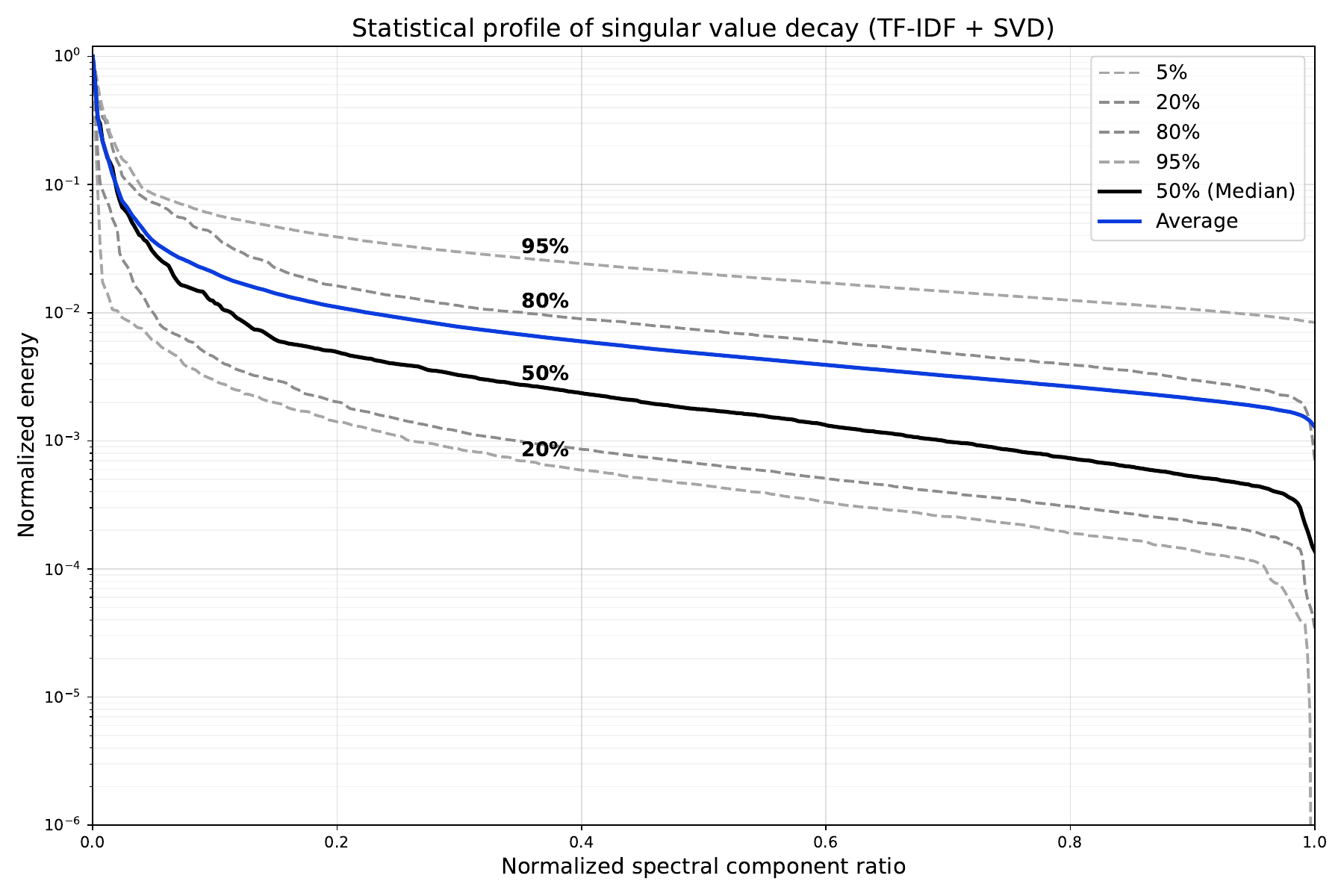}
	\caption{Aggregated spectral energy profiles ($\sigma_k^2 / \sigma_{\max}^2$) of term-document matrices from the 20 Newsgroups dataset. The sharp initial drop across all categories establishes that semantic information within individual topics is concentrated within a very low-dimensional latent subspace, leaving an uninformative tail of minor singular values.}
	\label{fig.text.spectrum}
\end{figure}

In summary, the empirical evidence demonstrates that realistic high-dimensional data are rarely strictly sparse, but instead feature a dominant structure obscured by a dense long tail of negligible coefficients. In practical inverse problems, these numerical fluctuations artificially inflate the discrete $\ell_0$ cardinality, yielding overly pessimistic theoretical bounds that do not reflect the true signal complexity. This discrepancy directly motivates moving beyond rigid binary counting. To adaptively discount uninformative tails while preserving meaningful signal mass, a continuous, scale-aware metric is required---a framework formally introduced in the next section via the effective number of nonzeros (ENZ).

 \section{Theory of the Effective Number of Nonzeros}\label{sec:effective_nonzeros}

\subsection{The ENZ Sparsity Measure}\label{sec:enz}

To formalize the distinction between literal support cardinality and effective informational sparsity, we introduce a continuous, magnitude-aware metric governed by the relative distribution of signal energy. For a non-zero vector $\bx \in \mathbb{R}^n$, we define its normalized magnitude distribution $\boldsymbol{\pi}(\bx)$ with coordinates:
\begin{equation}\label{eq:bpi}
    \pi_i(\bx) = \frac{|x_i|}{\|\bx\|_1}, \qquad i = 1,\dots,n.
\end{equation}
The associated Shannon entropy (measured in bits) is expressed as 
$$\mathcal{H}(\bx) = - \sum_{i=1}^n \pi_i(\bx)\log_2 \pi_i(\bx)$$
 with the standard convention $0\log_2 0 = 0$. The effective number of nonzeros (ENZ) is then defined by exponentiating this entropic core:
\begin{equation}
    \mathrm{ENZ}(\bx) := 2^{\mathcal{H}(\bx)}.
\end{equation}

The ENZ operates as a continuous generalization of the discrete $\ell_0$ norm. If $\bx$ is strictly $k$-sparse with equal magnitudes on its support, its normalized distribution becomes uniform, yielding $\mathcal{H}(\bx)=\log_2 k$ and $\mathrm{ENZ}(\bx)=k$. To characterize its robustness against noise, let $S$ be the dominant support of $\bx$ with $|S|=k$, and let $\delta := \sum_{i \in S^c} \pi_i  \ll 1$ represent the aggregate $\ell_1$ mass of the uninformative long tail. Let  $\mathbf{p}$, $\mathbf{q}$ be the localized conditional distributions on $S$ and $S^c$, respectively, i.e., 
\[ p_i = \pi_i/(1-\delta), \ i\in S; \quad q_i = \pi_i/\delta, \ i\in S^c.\]
By the classical grouping property of Shannon entropy, the total entropy exactly decomposes as:
\begin{equation*}\label{eq:entropy_decomposition}
   \begin{aligned}
    \mathcal{H}(\bx) = & \ - \sum_{i\in S}\pi_i\log_2 \pi_i - \sum_{i\in S^c}\pi_i \log_2 \pi_i \\
    = & \ - \sum_{i\in S}(1-\delta)p_i\log_2((1-\delta)p_i) - \sum_{i\in S^c}\delta q_i \log_2(\delta q_i) \\ 
    = & \ -(1-\delta)\log_2(1-\delta)\sum_{i\in S}p_i - (1-\delta) \sum_{i\in S}p_i\log_2 p_i - \delta \log_2 \delta - \delta q_i \log_2 q_i \\ 
    = & \ \mathcal{H}_b(\delta) + (1-\delta)\mathcal{H}(\mathbf{p}) + \delta\mathcal{H}(\mathbf{q}),
    \end{aligned}
\end{equation*}
where $\mathcal{H}_b(\delta) = - (1-\delta)\log_2(1-\delta)-\delta\log_2\delta $ is the binary entropy function, As the tail mass fraction vanishes ($\delta \to 0$), the binary uncertainty $\mathcal{H}_b(\delta)$ approaches zero and $(1-\delta)\mathcal{H}(\mathbf{p}) \to \mathcal{H}(\mathbf{p}) \le \log_2 k$. Crucially, the tail's entropic contribution is scaled linearly by its mass fraction, yielding the bound $\delta\mathcal{H}(\mathbf{q}) \le \delta\log_2(n-k)$, which vanishes rapidly even in high ambient dimensions $n$. Consequently, the total entropy remains strictly governed by the dominant support concentration, $\mathcal{H}(\bx) \approx \log_2 k$, ensuring that $\mathrm{ENZ}(\bx) \approx k$. This continuous structure allows the metric to track the effective dimensionality of the signal while remaining immune to expansive, low-energy perturbations that would otherwise cause the discrete $\ell_0$ count to explode \cite{cover1999elements, jaynes1957information}.

\begin{figure}[h]
\centering
\setlength{\tabcolsep}{2pt}
\begin{tabular}{ccc}
  \subfloat[$\ell_0$ unit ball]{\label{fig:l0}\includegraphics[width=0.285\textwidth]{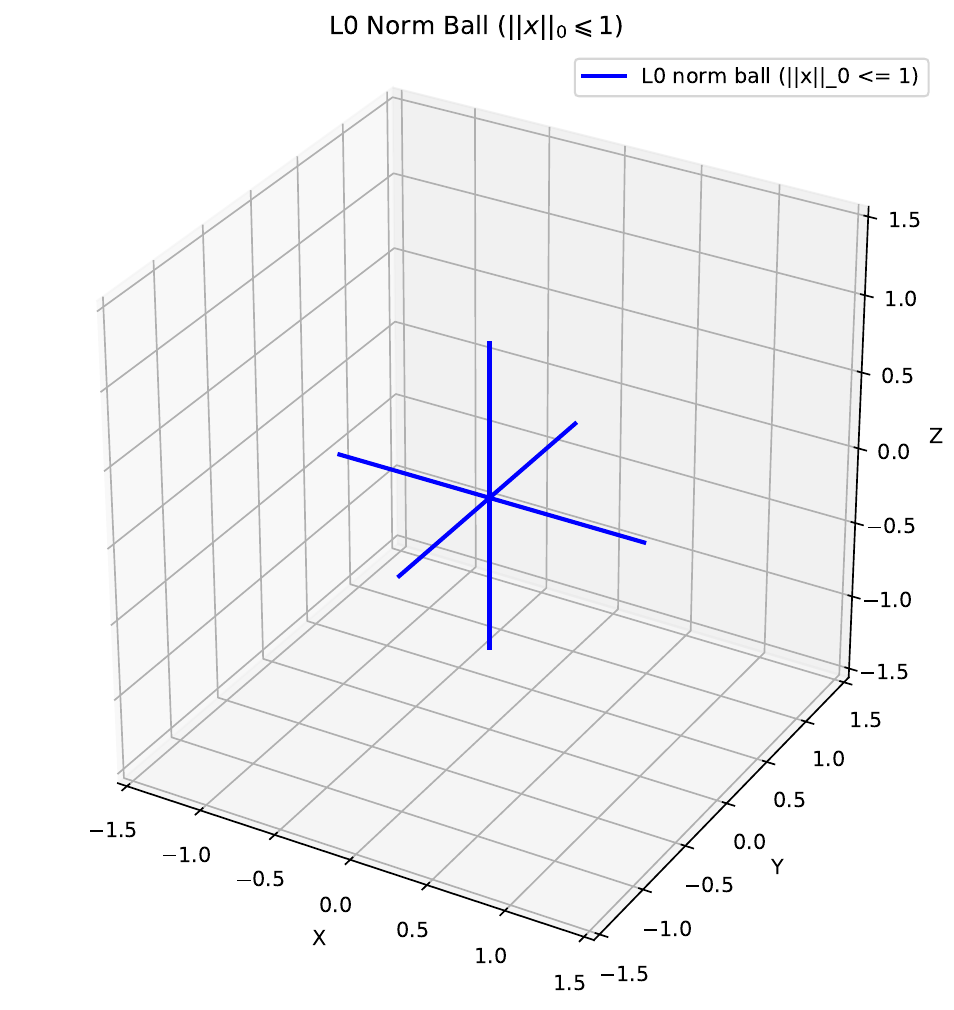}} &
  \subfloat[$\ell_1$ unit ball]{\label{fig:l1}\includegraphics[width=0.285\textwidth]{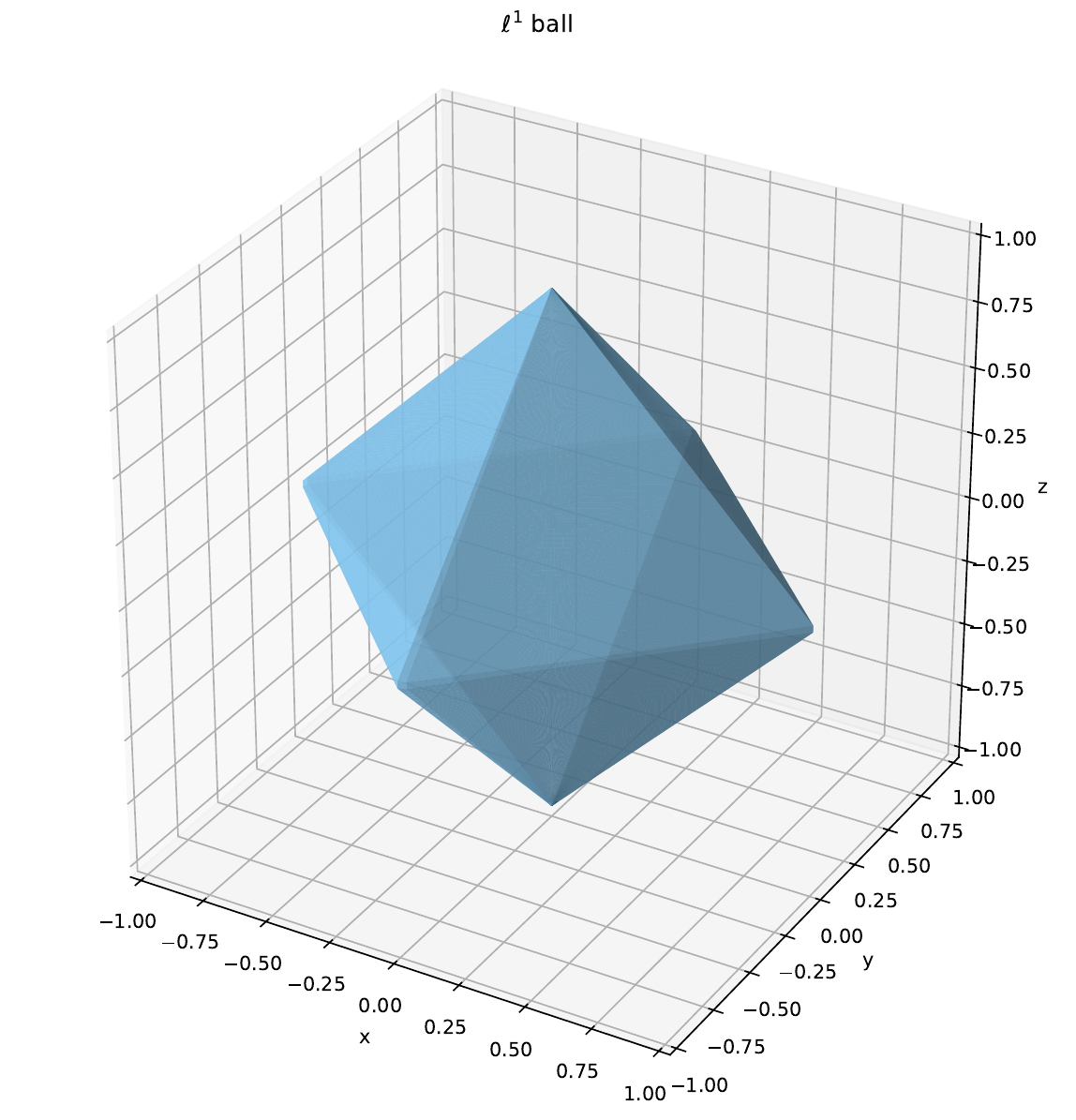}} &
  \subfloat[$\ell_{0.5}$ unit ball]{\label{fig:l05}\includegraphics[width=0.285\textwidth]{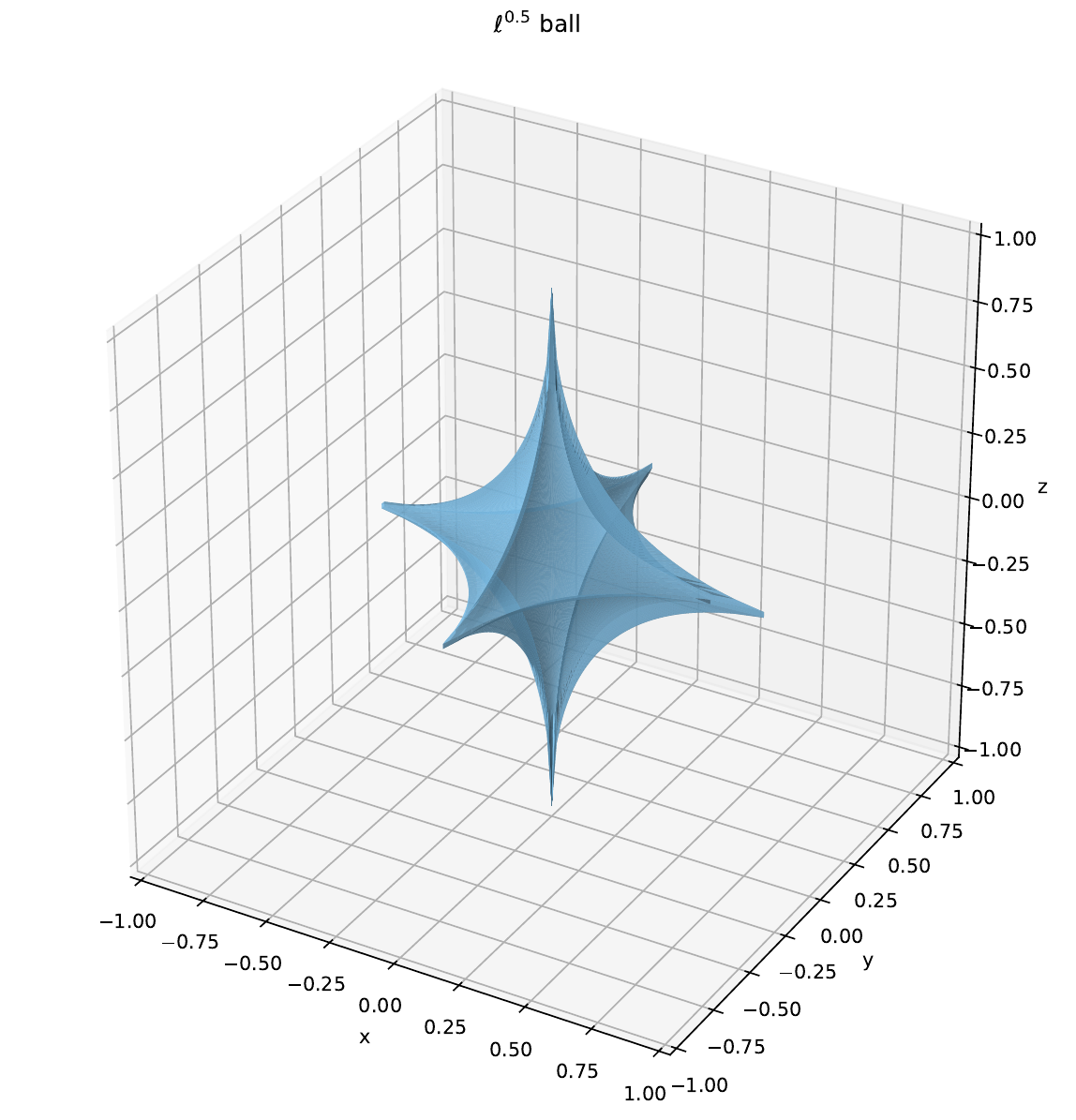}} \\
  \subfloat[$\ell_1-0.5\ell_2$ unit ball]{\label{fig:l1_2}\includegraphics[width=0.285\textwidth]{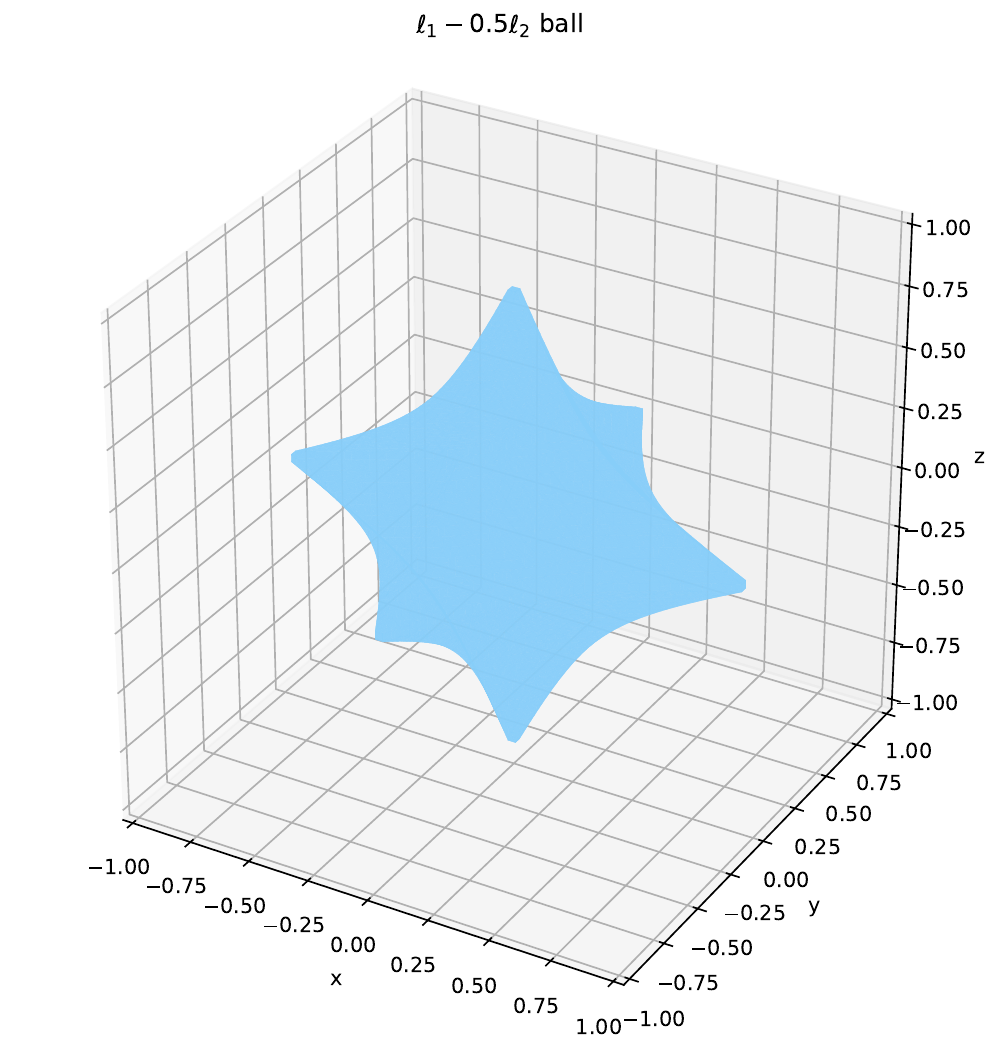}} &
  \subfloat[\textbf{ENZ $\le 1$}]{\label{fig:l1/2}\includegraphics[width=0.285\textwidth]{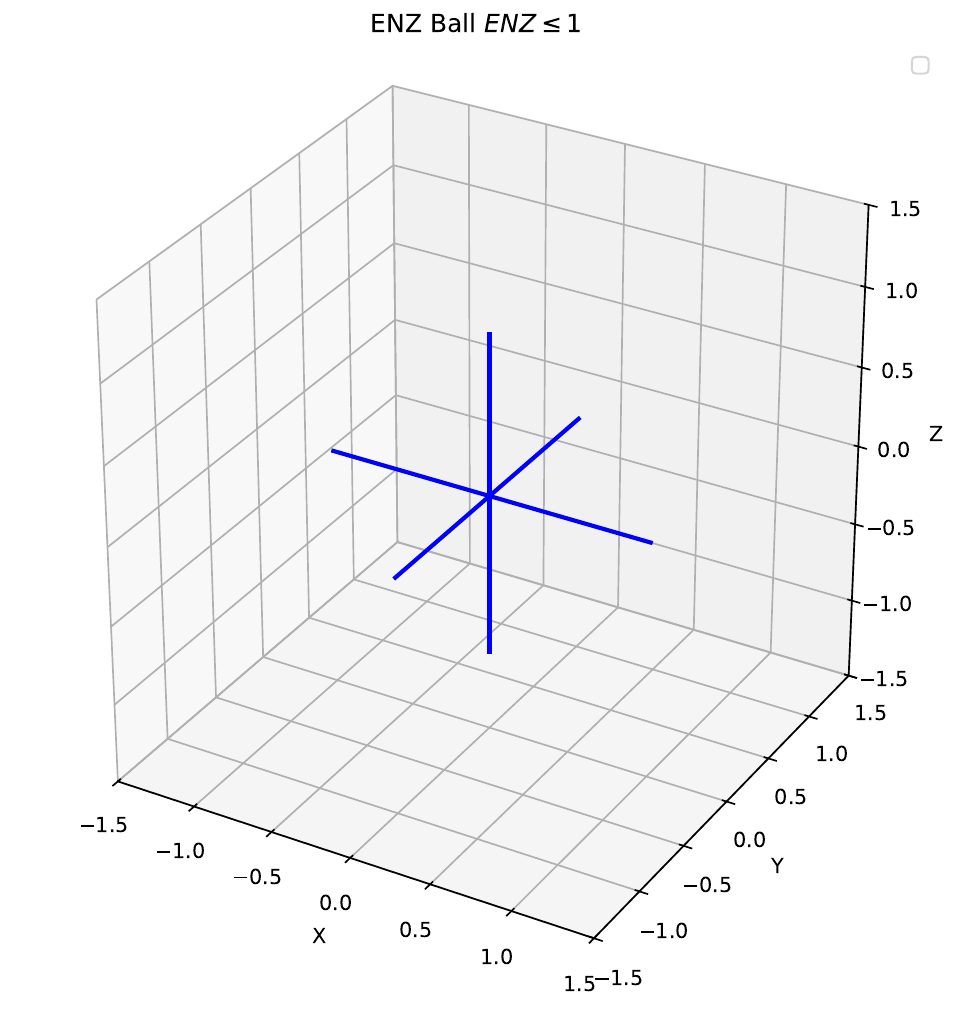}} &
  \subfloat[\textbf{ENZ $\le 1.5$}]{\label{fig:l1_ball_05}\includegraphics[width=0.285\textwidth]{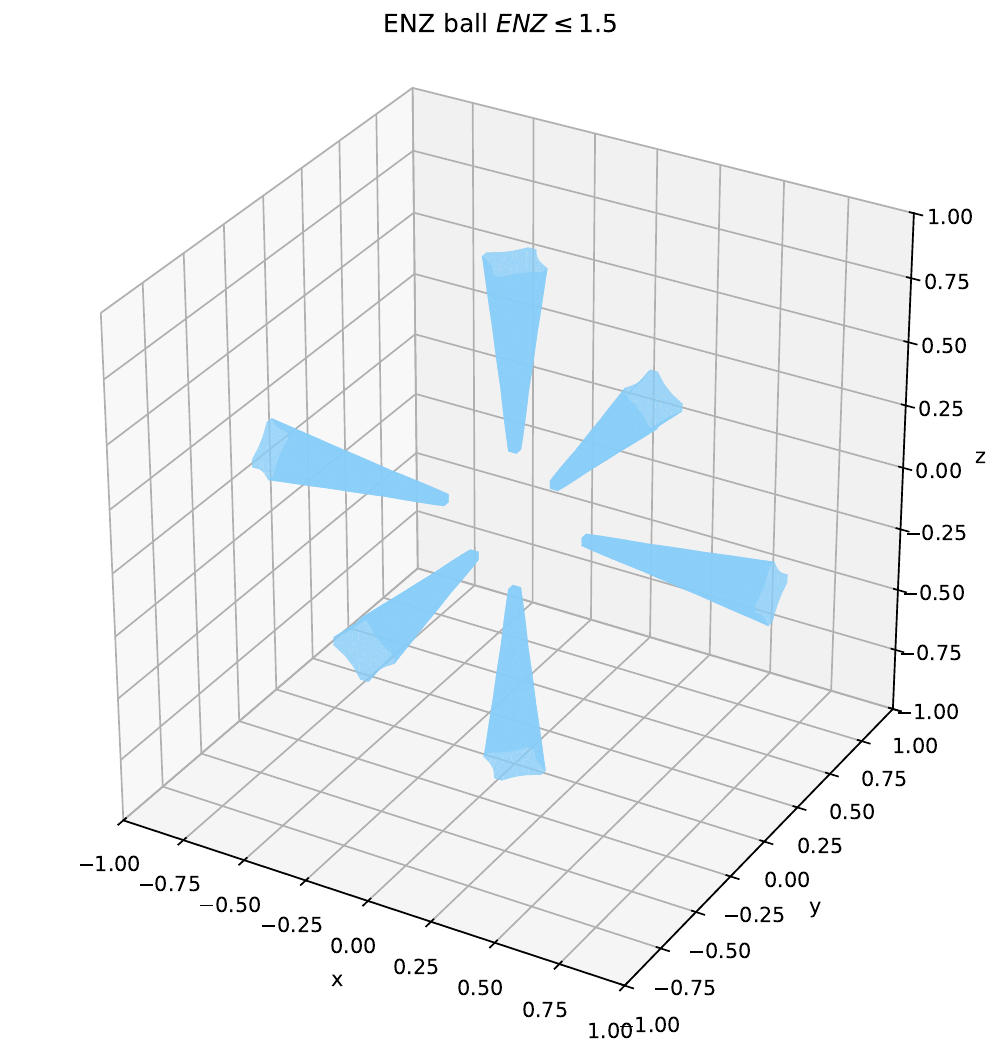}} 
\end{tabular}
\caption{Visualization of level sets in $\mathbb{R}^3$. The ENZ ball forms an unbounded, star-shaped geometry that asymptotically wraps around the coordinate axes, approximating the sparse axes without introducing the magnitude shrinkage bias characteristic of bounded $\ell_p$ balls.}
\label{fig:l012}
\end{figure}

The optimization properties of this regularizer are inherently tied to the geometry of its level sets. As shown in \Cref{fig:l012}, classical $\ell_p$ unit balls ($0 < p \le 1$) form closed, bounded neighborhoods around the origin. Although their sharp cusps along the axes promote sparsity, their bounded nature dictates that the regularizer penalty grows continuously with coordinate magnitudes, inevitably introducing a severe shrinkage bias that suppresses significant signal components during reconstruction.

In contrast, the level sets of the ENZ manifest as unbounded, star-shaped surfaces that asymptotically approximate the coordinate axes. Because the ENZ evaluates the normalized energy concentration $\boldsymbol{\pi}(\bx)$ rather than absolute lengths, it is inherently scale-invariant, allowing components on the effective support to grow arbitrarily large without incurring additional penalties. This geometry effectively decouples support identification from amplitude estimation, eliminating the structural shrinkage bias of classical $\ell_p$ penalties while providing a smooth, continuous proxy for the discrete $\ell_0$ cross.
 \subsection{Decomposition of the Effective Number of Nonzeros}
\label{sec:decomposition}

In this subsection, we derive an exact structural decomposition for the ENZ, revealing it as an information-theoretic softening of the combinatorial $\ell_0$ norm. This result demonstrates that the ENZ does not merely approximate support cardinality; rather, it mathematically discounts each nonzero component based on its relative energy concentration.

\begin{theorem}[Decomposition of Effective Nonzeros]
\label{thm:shannon_decomp}
Let $\bx \in \mathbb{R}^n \setminus \{0\}$, and let $\bpi$ be the normalized magnitude distribution where $\pi_i = |x_i|/\|\bx\|_1$. Let $S = \Supp(\bx)$ and denote $\bu$ as the uniform distribution over $S$, i.e., $u_i = 1/\|\bx\|_0$ for $i \in S$. The normalized Shannon entropy $\Hcal(\bx)$ admits the following decomposition:
\begin{equation}
    \Hcal(\bx) = \log_2 \|\bx\|_0 - D_{\mathrm{KL}}(\bpi \| \bu).
\end{equation}
Consequently, the effective number of nonzeros satisfies:
\begin{equation}
    \label{eq:shannon_exp_decomp}
    \mathrm{ENZ}(\bx) := 2^{\Hcal(\bx)} = \|\bx\|_0 \cdot 2^{-D_{\mathrm{KL}}(\bpi \| \bu)},
\end{equation}
where $D_{\mathrm{KL}}(\bpi \| \bu) \ge 0$ is the Kullback-Leibler divergence. Equality holds if and only if $|x_i|$ is constant for all $i \in S$.
\end{theorem}

\begin{proof}
By the definition of the Kullback-Leibler divergence restricted to the active support $S$:
\[
D_{\mathrm{KL}}(\bpi \| \bu) = \sum_{i \in S} \pi_i \log_2 \frac{\pi_i}{u_i} = \sum_{i \in S} \pi_i \log_2 \pi_i - \sum_{i \in S} \pi_i \log_2 \frac{1}{\|\bx\|_0}.
\]
Using the conservation property $\sum_{i \in S} \pi_i = 1$, we obtain $D_{\mathrm{KL}}(\bpi \| \bu) = -\Hcal(\bx) + \log_2 \|\bx\|_0$. Rearranging the terms and applying the exponential mapping yields \eqref{eq:shannon_exp_decomp}.
\end{proof}

\begin{figure}[ht]
\centering
\begin{tikzpicture}[
    node distance=1cm, 
    auto, 
    thick,
    block/.style={
        rectangle, 
        draw, 
        minimum width=2.5cm, 
        minimum height=1.2cm, 
        rounded corners=2pt,
        align=center,
        font=\small
    }
]
    \node[block, fill=pink!20] (enz) {{ENZ} \\ $2^{\Hcal(\bx)}$};
    \node[right=0.4cm of enz] (eq) {\scalebox{1.2}{$=$}};
    \node[block, fill=blue!20, right=0.4cm of eq] (l0) {{Support Size} \\ $\|\bx\|_0$};
    \node[right=0.4cm of l0] (times) {\scalebox{1.2}{$\times$}};
    \node[block, fill=red!20, right=0.4cm of times, minimum width=4cm] (kl) {{Distributional Efficiency} \\ $2^{-D_{\mathrm{KL}}(\bpi\|\bu)}$};

    \node [below=0.6cm of l0, text width=3.5cm, font=\footnotesize, color=darkgray, align=center] (l0_desc) 
        {Combinatorial count \\ (Integer-valued)};
        
    \node [below=0.6cm of kl, text width=5cm, font=\footnotesize, color=darkgray, align=center] (kl_desc) 
        {Information-theoretic discount \\ (Continuous in $(0, 1]$)};

    \draw[dashed, gray, ->, >=stealth] (l0.south) -- (l0_desc.north);
    \draw[dashed, gray, ->, >=stealth] (kl.south) -- (kl_desc.north);
\end{tikzpicture}
\caption{Structural decomposition of the ENZ. The metric factors the discrete support cardinality into a product with a continuous distributional efficiency term, bridging the gap between combinatorial and informational sparsity.}
\label{fig:enz_decomposition}
\end{figure} 

As visually summarized in \Cref{fig:enz_decomposition}, \Cref{thm:shannon_decomp} establishes that the ENZ operates as a continuous multiplicative adjustment to the discrete $\ell_0$ norm, rather than a magnitude-based coordinate shrinkage typical of $\ell_p$ regularizers. By defining the multiplicative factor $\eta(\bx) := 2^{-D_{\mathrm{KL}}(\bpi\|\bu)} \in (0,1]$ as the distributional efficiency of the active support, the metric cleanly interpolates between combinatorial support size and informational mass concentration. In practical inverse problems plagued by ambient fluctuations, the presence of a dense tail of negligible coefficients artificially inflates the integer count $\|\bx\|_0$, yet it forces the empirical distribution $\bpi$ far from the uniform state $\bu$, driving the divergence $D_{\mathrm{KL}}(\bpi\|\bu)$ upward and scaling $\eta(\bx)$ toward zero. Consequently, this informational discount adaptively downweights uninformative components without suppressing the true underlying signal profile. From an optimization standpoint, this behavior underscores a fundamental distinction: while classical $\ell_1$ or $\ell_p$ penalties penalize variables coordinate-wise, the entropic regularizer evaluates the global shape of the normalized amplitude spectrum, simultaneously penalizing support expansion and noise proliferation.

\begin{remark}[Hartley Entropy and the $\ell_0$ Limit]
The $\ell_0$ norm can be elegantly situated within this information-theoretic framework as a limiting case. In the literature, the {Hartley entropy} (or Rényi entropy of order zero) of a discrete distribution is defined as the logarithm of the support size \cite{hartley1928transmission}. For any $\bx \in \mathbb{R}^n$, we have:
\begin{equation*}
\Hcal_h(\bx) = \log_2 \|\bx\|_0, \qquad \text{and} \qquad \mathrm{ENZ}_0(\bx) = 2^{\Hcal_h(\bx)} = \|\bx\|_0.
\end{equation*}
This connection further confirms that our framework provides a unified view of sparsity. The classical $\ell_0$ norm represents a maximum-entropy scenario where all components are assumed equally significant, while ENZ refines this measure by accounting for the true informational concentration of the signal.
\end{remark}

This softening effect is geometrically illustrated on the 2-simplex $\Delta^2$ in \Cref{fig:simplex_heatmap}. The vertices represent 1-sparse signals where the Kullback-Leibler divergence is maximized relative to the 3-element support, forcing $\mathrm{ENZ}=1$. Conversely, the barycenter denotes the uniform distribution where $D_{\mathrm{KL}}(\bpi \| \bu) = 0$, yielding $\mathrm{ENZ}=3$. As the signal trajectory moves away from the uniform manifold toward lower-cardinality faces, the ENZ smoothly tracks the changing effective dimensionality, providing a well-conditioned optimization landscape for numerical sparse recovery.

\begin{figure}[ht]
\centering
\includegraphics[width=0.55\textwidth]{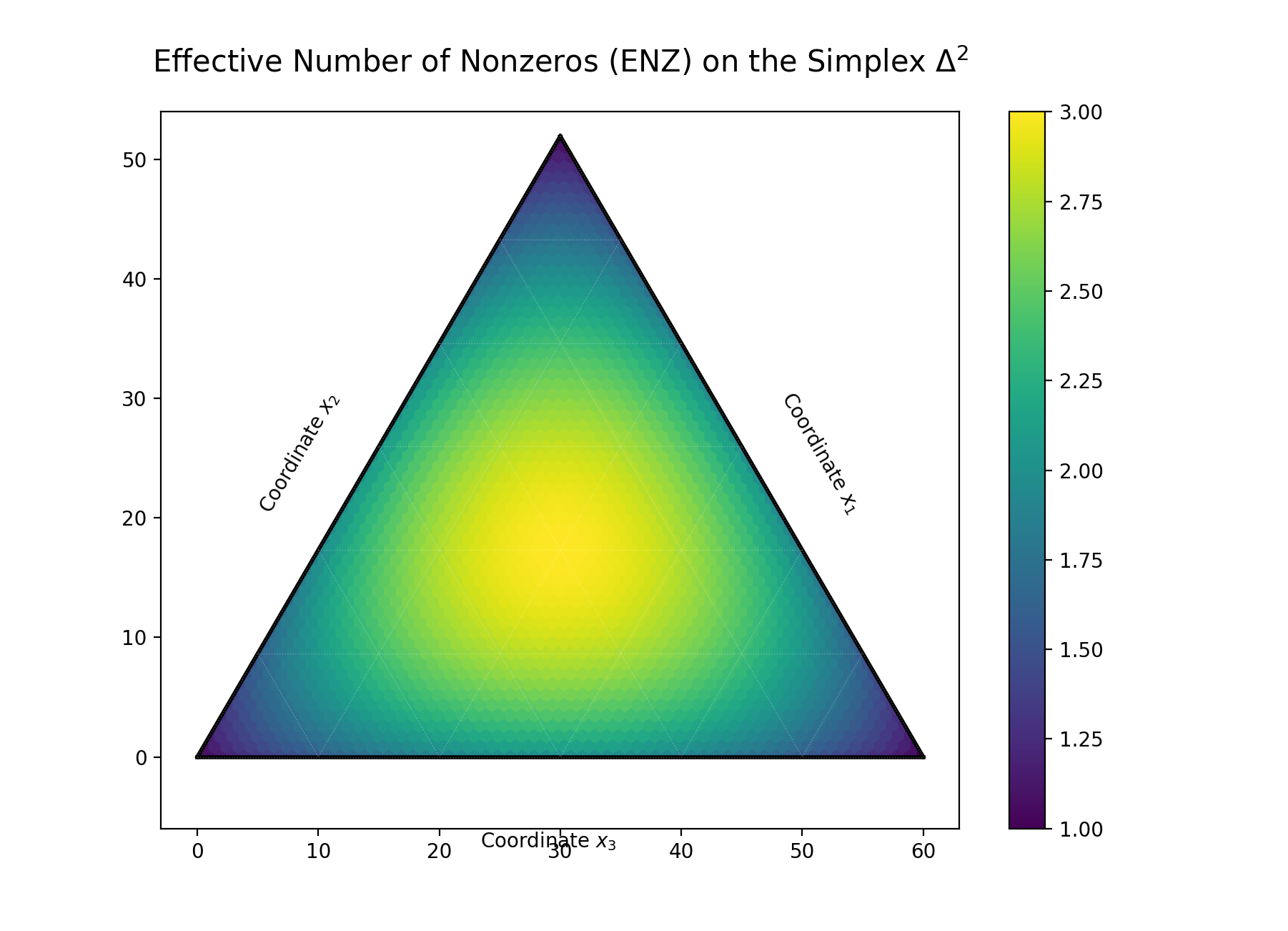}
\caption{Empirical landscape of the ENZ on the 2-simplex $\Delta^2$. The continuous gradient demonstrates how the entropic metric smoothly interpolates between lower-cardinality vertices and the uniform barycenter, providing a well-conditioned softening of the combinatorial $\ell_0$ cross.}
\label{fig:simplex_heatmap}
\end{figure}
\subsection{R\'enyi Extensions and Structural Hierarchy}\label{sec:entropy}

Having established a stability result for approximate sparsity, we next place the Shannon ENZ within a broader entropy-based family. While Shannon entropy provides a natural proxy for effective sparsity, it represents only a single point in the broader spectrum of information-theoretic functions. In this section, we generalize the construction to the R\'enyi entropy family and show how several scale-invariant sparsity measures arise as special cases. 

R\'enyi entropy introduces a tunable parameter $\alpha$, which allows for a controlled emphasis on different regions of the signal's magnitude distribution. This flexibility is useful for quantifying sparsity in high-dimensional settings where sensitivity to the tail or the peak of the distribution changes the behavior of the regularizer.

%
%


The R\'enyi entropy of order $\alpha$ ($\alpha > 0, \alpha \neq 1$) is defined as:
\begin{equation}
\Rcal_\alpha(\bx) = \frac{1}{1-\alpha} \log_2 \left( \sum_{i=1}^n \pi_i^\alpha \right).
\end{equation}
By varying $\alpha$, one can interpolate between different notions of concentration. Specifically, as $\alpha$ increases, the entropy becomes increasingly sensitive to the dominant components, whereas smaller values of $\alpha$ provide a more democratic assessment of the signal's support.

Analogous to the Shannon ENZ, we define the  R\'enyi Effective Number of Nonzeros (R\'enyi ENZ)  as:
\begin{equation}
\enz^R_\alpha(\bx) := 2^{\Rcal_\alpha(\bx)} = \left( \sum_{i=1}^n \pi_i^\alpha \right)^{\frac{1}{1-\alpha}}.
\end{equation}
This quantity provides a parameterized family of sparsity measures that inhabit the space between literal support cardinality and energy concentration.


The Rényi ENZ family is remarkably expressive, recovering several ``industry-standard'' sparsity measures as $\alpha$ takes specific values:
\begin{itemize}
    \item {The Shannon Limit ($\alpha \to 1$)}: Recovers the Shannon ENZ, $\enz(\bx) = 2^{\mathcal{H}(\bx)}$, providing the most balanced information-theoretic measure.
    \item {The $\ell_1/\ell_2$ Ratio ($\alpha = 2$)}: The second-order ENZ (Collision ENZ) corresponds exactly to the square of the $\ell_1/\ell_2$ ratio:
    \[ \enz^R_2(\bx) = \frac{\|\bx\|_1^2}{\|\bx\|_2^2}. \]
    This establishes a formal information-theoretic interpretation for the scale-invariant $\ell_1/\ell_2$ penalty widely used in sparse optimization.
    \item {The $\ell_1/\ell_\infty$ Ratio ($\alpha \to \infty$)}: In the limit of infinite order, the measure depends only on the peak magnitude, yielding $\enz^R_\infty(\bx) = \|\bx\|_1 / \|\bx\|_\infty$.
    \item {The $\ell_0$ Limit ($\alpha \to 0$)}: As the order approaches zero, the measure becomes indifferent to magnitudes, recovering the discrete cardinality
    \[
    \lim_{\alpha \to 0} \enz^R_\alpha(\bx) = \|\bx\|_0 .
    \]
\end{itemize}

\begin{figure}[h]
\centering
\begin{tikzpicture}[scale=1.2]
    \draw[->, thick] (0,0) -- (6,0) node[right] {Order $\alpha$};
    \draw[->, thick] (0,0) -- (0,4.5) node[above] {$\mathrm{ENZ}^R_\alpha(\bx)$};

    \draw[ultra thick, blue!80!black] plot [domain=0.1:5.5, samples=100] (\x, {1 + 3/(1+\x)});

    \filldraw [red] (0,4) circle (2pt);
    \node[left, red] at (0,4) {$\|\mathbf{x}\|_0$};
    
    \draw[dashed] (1,0) -- (1,2.5);
    \filldraw [black] (1,2.5) circle (1.5pt) node[above right] {Shannon $\mathrm{ENZ}$};
    \node[below] at (1,0) {$1$};

    \draw[dashed] (2,0) -- (2,2);
    \filldraw [black] (2,2) circle (1.5pt) node[above right] {$\|\mathbf{x}\|_1^2/\|\mathbf{x}\|_2^2$};
    \node[below] at (2,0) {$2$};

    \draw[dashed, gray] (0,1) -- (6,1);
    \node[right, gray] at (6,1) {$\alpha \to \infty: \|\mathbf{x}\|_1/\|\mathbf{x}\|_\infty$};

    \node[draw, fill=yellow!10, text width=4cm, align=center] at (4,3.5) {
        {Sparsity Hierarchy}\\
        $\alpha \uparrow \implies \mathrm{ENZ} \downarrow$
    };
\end{tikzpicture}
\caption{The Rényi ENZ spectrum as a function of order $\alpha$. The measure provides a continuous interpolation between the combinatorial $\ell_0$ norm ($\alpha \to 0$) and the scale-invariant ratio $\|\mathbf{x}\|_1/\|\mathbf{x}\|_\infty$ ($\alpha \to \infty$).}
\label{fig:renyi_spectrum}
\end{figure}

Crucially, the Rényi ENZ satisfies a {monotonicity property} with respect to $\alpha$. For any signal $\bx$, we have the following fundamental hierarchy of effective sparsity:
\begin{equation}
\|\bx\|_0 \ge \enz_{\text{Shannon}}(\bx) \ge \frac{\|\bx\|_1^2}{\|\bx\|_2^2} \ge \frac{\|\bx\|_1}{\|\bx\|_\infty}.
\end{equation}
The hierarchical structure of the Rényi ENZ is visually encapsulated in \Cref{fig:renyi_spectrum}. As the order $\alpha$ increases, the ENZ spectrum exhibits a strictly {monotonic decreasing property}, effectively ``scanning'' the signal from its combinatorial support toward its peak concentration. 
%
%


The relationship between Rényi ENZ and the $\ell_0$ norm is characterized by the following generalization of the decomposition theorem.

\begin{theorem}[Generalized Rényi Decomposition]\label{thm:renyi_decomp}
Let $\bx \in \mathbb{R}^n \setminus \{0\}$ and let $\bu$ be the uniform distribution over $S = \Supp(\bx)$. For any $\alpha > 0, \alpha \neq 1$, the Rényi entropy admits the decomposition:
\begin{equation}\label{eq:renyi_decomposition}
\Rcal_\alpha(\bx) = \log_2 \|\bx\|_0 - D_\alpha(\bpi \| \bu),
\end{equation}
where $D_\alpha(\bpi \| \bu) = \frac{1}{\alpha-1} \log_2 (\sum_{i \in S} \pi_i^\alpha u_i^{1-\alpha})$ is the Rényi divergence of order $\alpha$. Consequently, the Rényi ENZ satisfies the multiplicative relation:
\begin{equation}
\enz^R_\alpha(\bx) = \|\bx\|_0 \cdot 2^{-D_\alpha(\bpi \| \bu)} \le \|\bx\|_0.
\end{equation}
\end{theorem}

\section{Stability of ENZ under Restricted Isometry}\label{sec:stability}

The preceding section defines ENZ and characterizes its decomposition into support size and distributional imbalance. We now ask how the dominant part of an approximately sparse signal behaves under noisy linear measurements. Traditional recovery guarantees in compressed sensing are often predicated on exact $k$-sparsity, a condition that is frequently violated in practical applications due to numerical noise and modeling inaccuracies. We establish a restricted-isometry stability bound showing that the dominant-support discrepancy is controlled by the measurement noise and the energy of the negligible signal tail.

\subsection{Problem Setup and Notation}

Consider a sensing matrix $A \in \mathbb{R}^{m \times n}$ and two signals $\bx, \by \in \mathbb{R}^n$ that satisfy the linear constraint $A\bx = A\by$. We assume $\bx$ and $\by$ are effectively $k$-sparse. To facilitate the analysis, we define the following sets and vectors:
\begin{itemize}
    \item $S_x :=$ indices of the $k$ largest (in magnitude) entries of $\bx$, with $S_x^c$ as its complement.
    \item $S_y :=$ indices of the $k$ largest (in magnitude) entries of $\by$, with $S_y^c$ as its complement.
    \item $T := S_x \cup S_y$ represents the union of the dominant supports, satisfying $|T| \le 2k$ with $T^c$ as its complement.
\end{itemize}
The difference vector is denoted by $\bh := \by - \bx \in \mathbb{R}^n$. For any index set $I \subset \{1,\dots,n\}$, we define $\bh_I$ as the sub-vector consisting of components indexed by $I$ and zeros elsewhere. This allows the decomposition:
\[
A \bh = A \bh_T + A \bh_{T^c}.
\]

\subsection{Restricted Isometry Property (RIP)}
The stability of recovery hinges on the Restricted Isometry Property (RIP) of order $2k$. Specifically, we assume $A$ satisfies:
\begin{equation}
(1-\delta_{2k}) \|\bu\|_2^2 \le \|A \bu\|_2^2 \le (1+\delta_{2k}) \|\bu\|_2^2, \quad \forall\, 2k\text{-sparse } \bu.
\label{eq:2k-RIP}
\end{equation}
where $\delta_{2k}\in(0,1)$ denotes the (smallest) restricted isometry constant of order $2k$ associated with $A$, i.e., the smallest value such that \eqref{eq:2k-RIP} holds for all $2k$-sparse vectors.

The following proposition summarizes essential properties derived from the $2k$-RIP.  The first inequality is a direct consequence of the \emph{restricted eigenvalue (RE) condition} \cite{BickelRitovTsybakov2009} or the \emph{sparse eigenvalue condition} \cite{MeinshausenYu2009,BuneaTsybakovWegkamp2007}, which are used to establish oracle inequalities and recovery guarantees for the Lasso and related estimators.   Part (ii) is from \cite{decoding} and \cite[Lemma 2.1]{CRMATH}.

\begin{proposition}\label{prop.cross} Suppose $A\in\mathbb{R}^{m\times n}$ satisfies the $2k$-RIP \eqref{eq:2k-RIP}.  
\begin{enumerate}
\item[(i)] For every index set $T$ with $|T| \le 2k$, the restricted matrix $A_T$ satisfies the lower singular value bound:
\begin{equation}\label{eq:sigma}
\sigma_{\min}(A_T) \ge \sqrt{1-\delta_{2k}} > 0. 
\end{equation}
\item[(ii)] Let \(\bu,\bv\in\mathbb{R}^n\) have disjoint supports such that $|\Supp(\bu)| + | \Supp(\bv)| \le 2k$. Then:
\[ |\langle A\bu,A\bv\rangle|\le \delta_{2k}\,\|\bu\|_2\|\bv\|_2. \]
\end{enumerate}
\end{proposition} 

\subsection{Stability Analysis under Noisy Measurements}
To formalize effective sparsity in the noisy setting, we model each signal by its top-$k$ dominant entries plus a small tail. In particular, we quantify the tail energy outside the top-$k$ index sets $S_x$ and $S_y$ by the following conditions:
\begin{equation}\label{eq:tails}
\|\bx_{S_x^c}\|_2 \le \varepsilon_x, \qquad \|\by_{S_y^c}\|_2 \le \varepsilon_y.
\end{equation}
We next consider the noisy observation model $A \bx + \bm{\eta} = A \by + \bm{\eta}'$, where $\bm{\eta}, \bm{\eta}' \in \mathbb{R}^m$ represent measurement noise or modeling errors. Defining the effective noise as $\be := \bm{\eta} - \bm{\eta}'$, the discrepancy vector satisfies $A \bh = \be$.

Our goal is to show that if $A$ is sufficiently well-conditioned on all $2k$-column submatrices—formalized via the RIP or a lower singular value bound—then the dominant coordinates of $\bx$ and $\by$ must remain close. This extends the classical uniqueness results of sparse recovery to the regime of approximately sparse vectors \cite{blanchard2011compressed, candes2006robust, donoho2006compressed, candes2008introduction}.

 \Cref{thm:noisy-stability} formalizes this intuition. It demonstrates that as long as $A$ does not collapse any $2k$-sparse vector, any two signals with $k$ effective nonzeros will possess similar values on their primary supports. Crucially, this stability is maintained even if the literal $\ell_0$ norm of the signals significantly exceeds $k$. The resulting error bound is explicitly controlled by the aggregate tail energy $\varepsilon_x + \varepsilon_y$ and the restricted condition number of $A$. This confirms a key trade-off: while effective nonzeros provide robustness to small perturbations, they necessitate a measurement matrix that is stable across the entire effective support union.
 

\begin{theorem}[Stability of the Effective Nonzeros  under Noise]
\label{thm:noisy-stability}
Let $A \in \mathbb{R}^{m \times n}$ satisfy the $2k$-RIP \eqref{eq:2k-RIP} with constant $\delta_{2k}\in(0,1)$. 
Let $\bx, \by \in \mathbb{R}^n$ be two signals, and define:
\begin{itemize}
    \item $S_x, S_y \subset \{1,\dots,n\}$: the index sets of the $k$ largest-magnitude entries of $\bx$ and $\by$, respectively;
    \item $T := S_x \cup S_y$ (so $|T| \le 2k$) and $\bh := \by - \bx$.
\end{itemize}
Assume the tail conditions
\[
\|\bx_{S_x^c}\|_2 \le \varepsilon_x, \qquad \|\by_{S_y^c}\|_2 \le \varepsilon_y,
\]
and suppose there exists $\be \in \mathbb{R}^m$ such that $A\bh = \be$ (i.e., $A\bx$ and $A\by$ differ by $\be$).
Under the assumptions above, the difference between $\bx$ and $\by$ on the union of
their top-$k$ supports satisfies 
\begin{equation}\label{eq.noisebound1} 
\|\bh_T\|_2
\le \frac{\sqrt{2(n-2k)}\,\delta_{2k}}{(1-\delta_{2k})\sqrt{k}}\,(\varepsilon_x+\varepsilon_y)  + \frac{\sqrt{1+\delta_{2k}}}{{(1-\delta_{2k}) }}\, \|\be\|_2.
\end{equation}
Consequently, the effective nonzeros of $\bx$ and $\by$ satisfy
\begin{equation}\label{eq.noisebound2} 
\|\bx_{S_x} - \by_{S_y}\|_2
\le \big(\frac{\sqrt{2(n-2k)}\,\delta_{2k}}{(1-\delta_{2k})\sqrt{k}}+1\big)\,(\varepsilon_x+\varepsilon_y)   + \frac{\sqrt{1+\delta_{2k}}}{{(1-\delta_{2k})}}\, \|\be\|_2.
\end{equation}
 \end{theorem}

\begin{proof} Starting from $A \bh = A(\by-\bx) = \be$ and $A \bh_T = - A \bh_{T^c} + \be.$ 
 We further decompose $\bh_T$ as
\[
\bh_T = \bu + \bv,
\qquad
\bu := \bh_{S_x}, \quad \bv := \bh_{S_y\setminus S_x}.
\]
Observe that \(\bu\) and \(\bv\) are at most \(k\)-sparse and have disjoint supports. It follows that \(\|\bh_T\|_2^2 = \|\bu\|_2^2 + \|\bv\|_2^2\).
Next, partition the tail $\bh_{T^c}$ into blocks $B_1, B_2, \dots$ of size at most $k$, sorted in decreasing order of magnitude. By standard sorted-tail inequalities, for any $j \ge 2$, 
\[\|\bh_{B_j}\|_2 \le \frac{\|\bh_{B_{j-1}}\|_1}{\sqrt{k}}.\]
Summing up both sides yields 
\begin{equation}\label{eq.tail} 
\sum_{j\ge 1} \|\bh_{B_j}\|_2 \le \frac{\|\bh_{T^c}\|_1}{\sqrt{k}}.
\end{equation}

Fix any block \(B_j\subseteq T^c\) with \(|B_j|\le k\). Since each of
\(\Supp(\bu)\) and \(\Supp(\bv)\) has size \(\le k\)
and is disjoint from \(B_j\), the standard RIP-based cross-term bound   \Cref{prop.cross}(ii) 
applies to the pairs \((\bu,\bh_{B_j})\) and \((\bv,\bh_{B_j})\):
\[
|\langle A \bu,\, A \bh_{B_j}\rangle| \le \delta_{2k}\,\|\bu\|_2 \,\|\bh_{B_j}\|_2,
\qquad
|\langle A \bv,\, A \bh_{B_j}\rangle| \le \delta_{2k}\,\|\bv\|_2 \,\|\bh_{B_j}\|_2.
\]
Summing yields
\[
|\langle A \bh_T,\, A \bh_{B_j}\rangle|
= |\langle A(\bu+\bv),\, A \bh_{B_j}\rangle|
\le \delta_{2k}\,(\|\bu\|_2+\|\bv\|_2)\,\|\bh_{B_j}\|_2.
\]
Using \(\|\bu\|_2+\|\bv\|_2 \le \sqrt{2}\,\sqrt{\|\bu\|_2^2+\|\bv\|_2^2} = \sqrt{2}\,\|\bh_T\|_2\),
we obtain the   bound
\begin{equation}\label{eq:cross_bound_final}
|\langle A \bh_T,\, A \bh_{B_j}\rangle|
\le \sqrt{2}\,\delta_{2k}\,\|\bh_T\|_2\,\|\bh_{B_j}\|_2.
\end{equation}

 On the other hand, from \(A \bh = \be\) (i.e.\ \(A \bh_T + A \bh_{T^c}= \be\)),  we have
\[
A \bh_T = -\sum_{j\ge 1} A \bh_{B_j} +  \be.
\]
Take the squared norm and apply the RIP lower bound on \(\bh_T\) (note
\(|\Supp(\bh_T)|\le 2k\)):
\begin{align}
\label{eq.hBj}
(1-\delta_{2k})\|\bh_T\|_2^2
&\le \|A \bh_T\|_2^2 \notag\\
&= \Big\langle A \bh_T,\;
-\sum_{j\ge 1} A \bh_{B_j} + \be\Big\rangle \notag\\
&=
-\sum_{j\ge 1}
\langle A \bh_T,\, A \bh_{B_j}\rangle
+ \langle A\bh_T,  \be\rangle .
\end{align} 
It follows from  \eqref{eq:cross_bound_final} and the triangle inequality that 
\begin{align*}
(1-\delta_{2k})\|\bh_T\|_2^2
&\le
\sum_{j\ge 1}
|\langle A \bh_T,\, A \bh_{B_j}\rangle|
+ | \langle A\bh_T,  \be\rangle|\\
&\le
\sqrt{2}\,\delta_{2k}\,\|\bh_T\|_2
\sum_{j\ge 1} \|\bh_{B_j}\|_2
+ | \langle A\bh_T,  \be\rangle| .
\end{align*}
Dividing both sides by \((1-\delta_{2k})\|\bh_T\|_2\) (the case \(\bh_T=0\) is trivial), 
we obtain from \eqref{eq.tail} 
\begin{align}
\label{eq.tail2}
\|\bh_T\|_2^2
&\le
\frac{\sqrt{2}\,\delta_{2k}}{1-\delta_{2k}}
\sum_{j\ge 1} \|\bh_{B_j}\|_2
+
\frac{| \langle A\bh_T,  \be\rangle|}
{(1-\delta_{2k})\|\bh_T\|_2} \notag\\
&\le
\frac{\sqrt{2}\,\delta_{2k}}{1-\delta_{2k}}
\frac{\|\bh_{T^c}\|_1}{\sqrt{k}}
+
\frac{| \langle A\bh_T,  \be\rangle|}
{(1-\delta_{2k})\|\bh_T\|_2}.
\end{align} 
 By Cauchy-Schwarz inequality and  the lower RIP bound, we have 
\begin{align*} 
\|\bh_{T^c}\|_1
&\le \|\bx_{T^c}\|_1 + \|\by_{T^c}\|_1  \\
&\le \sqrt{|T^c|}
\bigl(\|\bx_{T^c}\|_2+\|\by_{T^c}\|_2\bigr)
\le \sqrt{n-2k}\,(\varepsilon_x+\varepsilon_y),\\
|\langle A \bh_T, \be \rangle|
&\le \|A \bh_T\|_2 \|\be\|_2
\le \sqrt{1+\delta_{2k}}\,
\|\bh_T\|_2\|\be\|_2 .
\end{align*}
 Plugging in \eqref{eq.tail2},  we obtain \eqref{eq.noisebound1}.  
 
Finally, to compare the effective  $k$ nonzeros, we write:
\[
\bx_{S_x} - \by_{S_y}
=
(\bx-\by)_T
+
\bx_{S_y \setminus S_x}
-
\by_{S_x \setminus S_y},
\]
Note that $S_y \setminus S_x \subset S_x^c$ and $S_x \setminus S_y \subset S_y^c$. Therefore, the additional terms are bounded by the tail norms $\varepsilon_x$ and $\varepsilon_y$ respectively. Applying the triangle inequality and \eqref{eq.noisebound1} completes the proof.
\end{proof}

 \Cref{thm:noisy-stability} provides a formal justification for the use of effective nonzeros in sparse recovery.  Classical RIP requirements typically scale with the apparent sparsity. Attempting to protect every infinitesimal nonzero component forces $\delta_t$ conditions for $t \gg 2k$, which are harder to satisfy and require significantly more measurements.  \Cref{thm:noisy-stability} demonstrates that as long as the matrix is well-conditioned on the $2k$-dominant subspace, stable recovery is guaranteed despite the presence of pervasive ``speckle'' noise or numerical artifacts. This ensures that theoretical conditions and sampling resources are dedicated solely to the aspects of the solution that truly matter for interpretation and computation.

  \section{ENZ-Based Recovery and Regularization Formulations}
\label{sec:entropy-model}

The theoretical guarantee  \Cref{thm:noisy-stability} established in  \Cref{sec:stability} demonstrates that stable recovery under the RIP hinges crucially on controlling the aggregate energy of the uninformative tail coefficients, $\varepsilon_x$ and $\varepsilon_y$ \cite{CRMATH, candes2006stable}. In practical inverse problems where the true underlying support is unknown, seeking a strictly $k$-sparse solution via combinatorial optimization forces the numerical solver to aggressively safeguard negligible fluctuations, which leads to overly pessimistic sampling requirements \cite{natarajan1995sparse, lopes2013estimating}. To operationalize the stability insight of \Cref{thm:noisy-stability}, one must transition from binary cardinality counting to a variational framework that systematically discounts these diffuse, long-tail perturbations according to their continuous energy distribution \cite{hurley2009comparing}. This provides the foundational motivation for formulating an unconstrained regularization paradigm governed by the effective number of nonzeros (ENZ).

Ideally, an unconstrained sparse recovery problem that directly minimizes the effective dimensionality of the signal can be formulated as:
\begin{equation}\label{prob.main}
\min_{\bx \in \mathbb{R}^n} \ \mathrm{ENZ}(\bx)
\quad \text{s.t.} \quad
\| {A}\bx - \bb\|_2^2 \le \varepsilon,
\end{equation}
where $\varepsilon > 0$ represents a prescribed noise tolerance or an estimated upper bound on the residual energy of the measurements \cite{donoho2006compressed}. However, the highly nonconvex and fractional nature of $\mathrm{ENZ}(\bx)$ makes problem \eqref{prob.main} optimizationally intractable in its raw form. Because the exponentiation mapping is strictly monotonic, minimizing $\mathrm{ENZ}(\bx)$ on the simplex is mathematically equivalent to minimizing its information-theoretic core, the Shannon entropy $\mathcal{H}(\bx)$ \cite{shannon1948mathematical}. By casting the data fidelity constraint into a penalized objective, we establish the core ENZ-based entropy regularization problem:
\begin{equation}\label{prob.main2}
\min_{\bx \in \mathbb{R}^n} \ \Phi(\bx) := \frac{1}{2}\| {A}\bx - \bb\|_2^2 + \lambda \mathcal{H}(\bx),
\end{equation}
where $\lambda > 0$ is a regularization parameter that explicitly balances spatial data fidelity with the informational sparsity of the recovered coefficient vector \cite{jaynes1957information}.

\subsection{Computational Motivation for Unnormalized Shannon Entropy}

Although the regularization problem \eqref{prob.main2} directly mirrors the geometry of the ENZ model, the explicit dependence of $\mathcal{H}(\bx)$ on the global $\ell_1$-normalization introduces severe analytical and numerical hurdles that impede large-scale computation. Specifically, the probability coordinate mapping $\pi_i(\bx) = |x_i| / \|\bx\|_1$ creates an all-to-all coupling among the components of $\bx$ \cite{roy2007effective}. This global coupling destroys the separability of the regularizer, preventing the derivation of coordinate-wise updates or simple proximal operators, and rendering standard proximal-splitting algorithms computationally prohibitive \cite{beck2009fast}.

To circumvent this structural bottleneck while preserving the scale-awareness necessary for effective sparsity, we introduce an unnormalized Shannon entropy surrogate that completely decouples the coordinates. Let the signal be bounded within a localized envelope such that $\|\bx\|_\infty \le C$ for a positive scaling constant $C$. By replacing the coupled distribution $\pi_i(\bx)$ with the independent normalized magnitudes $|x_i|/C \in [0,1]$, we define the unnormalized Shannon entropy $\mathcal{H}_u(\bx; C)$ as:
\begin{equation}
    \mathcal{H}_u(\bx; C)
    := -\sum_{i=1}^n \frac{|x_i|}{C} \log_2 \frac{|x_i|}{C} + \frac{\|\bx\|_1}{C}
    = \frac{1}{C}\Big( -\sum_{i=1}^n |x_i| \log_2 |x_i| + (1+\log_2 C)\|\bx\|_1 \Big).
    \label{eq:unnormalized_entropy}
\end{equation}
The analytical connection between the standard normalized entropy and this unnormalized surrogate is established via the multiplicative relation:
\begin{equation}
    \mathcal{H}_u(\bx; C) = \frac{\|\bx\|_1}{C} \left( \mathcal{H}(\bx) + 1 + \log_2 \frac{C}{\|\bx\|_1} \right).
\end{equation}
Crucially, in the limiting scenario where the scaling bound tightly matches the energy scale of the iterate, i.e., $C = \|\bx\|_1$, the relation collapses to a linear shift: $\mathcal{H}_u(\bx; C) = \mathcal{H}(\bx) + 1$. Consequently, minimizing the separable formulation $\mathcal{H}_u(\bx; C)$ serves as a mathematically consistent proxy for minimizing the original coupled problem \eqref{prob.main2}. The component function $h(z) = -z \log_2 z + z$ is smooth and strictly concave on $\mathbb{R}_{++}$, meaning that $\mathcal{H}_u$ inherits positive homogeneity of degree one, $\mathcal{H}_u(c\bx) = c\mathcal{H}_u(\bx)$ for any $c>0$, which stabilizes the optimization landscape against arbitrary amplitude scaling.

\subsection{Algorithmic Implementation}
\label{sec:algorithmic-implementation}

Equipped with the separable unnormalized Shannon surrogate \eqref{eq:unnormalized_entropy}, we now develop an efficient, large-scale numerical algorithm to solve the unconstrained penalized formulation:
\begin{equation}\label{eq:unconstrained_objective}
    \min_{\bx \in \mathbb{R}^n} \ F(\bx) := \frac{1}{2}\|A\bx - \bb\|_2^2 + \lambda \mathcal{H}_u(\bx; C),
\end{equation}
where for computational efficiency and gradient consistency, the log-base is shifted to the natural logarithm $\ln$, effectively absorbing a constant scaling factor into $\lambda$. The optimization of \eqref{eq:unconstrained_objective} presents two clear numerical challenges: the regularizer is non-smooth due to the absolute value operator $|x_i|$, and the choice of the scaling parameter $C$ must adaptively track the moving energy scale of the solution trajectory without re-coupling the coordinates during inner optimization steps.

To resolve these issues, we propose a double-loop numerical framework. In the outer loop, we deploy an iterative re-scaling fixed-point scheme where the scaling parameter is frozen to the $\ell_2$ norm of the previous iterate, setting $C^{(t)} = \|\bx^{(t)}\|_2$ at outer iteration $t$. This keeps the regularizer strictly decoupled during the inner optimization phase. Concurrently, to handle the non-smoothness at the origin, we construct a smooth approximation by replacing $|x_i|$ with $\phi_\epsilon(x_i) = \sqrt{x_i^2 + \epsilon}$, where $\epsilon > 0$ is a temporary smoothing parameter managed by a decreasing homotopy continuation schedule.

For a given scale parameter $C > 0$ and a smoothing parameter $\epsilon > 0$, the smoothed unnormalized entropic objective function is defined generally as:
\begin{equation}
    F_{\epsilon}(\bx; C) = \frac{1}{2}\|A\bx - \bb\|_2^2 - \frac{\lambda}{C} \sum_{i=1}^n \sqrt{x_i^2+\epsilon} \ln\left(\frac{\sqrt{x_i^2+\epsilon}}{C}\right) + \frac{\lambda}{C} \sum_{i=1}^n \sqrt{x_i^2+\epsilon}.
\end{equation}
Because $C$ and $\epsilon$ are treated as static constants within the inner optimization phase, $F_{\epsilon}(\bx; C)$ is continuously differentiable, allowing its exact gradient to be evaluated analytically. We minimize this well-conditioned subproblem using the limited-memory Broyden-Fletcher-Goldfarb-Shanno (L-BFGS) algorithm \cite{liu1989limited,nocedal2006numerical}, which achieves linear memory complexity by utilizing historical gradient updates to approximate curvature steps.

To eliminate the smoothing bias and ensure scale alignment, our overall framework adopts a double-loop continuation strategy. At each outer iteration $t$, the inner loop minimizes $F_{\epsilon}(\bx; C)$ with the parameters frozen at the current outer states $C = C^{(t)}$ and $\epsilon = \epsilon^{(t)}$. Upon inner convergence, the outer loop executes a fixed-point update on the scale factor via $C^{(t+1)} = \|\bx^{(t+1)}\|_2$ and reduces the smoothing parameter through a geometric decay factor $\gamma \in (0,1)$. The complete computational procedure is summarized in \Cref{alg:enz_lbfgs}.

\begin{algorithm}[H]
\caption{Iterative Re-scaling and Continuation L-BFGS for ENZ Recovery}
\label{alg:enz_lbfgs}
\begin{algorithmic}[1]
\REQUIRE Matrix $A \in \mathbb{R}^{m \times n}$, vector $\bb \in \mathbb{R}^m$, regularizer $\lambda > 0$, decay $\gamma \in (0,1)$, tolerances $\tau_{in}, \tau_{out}$.
\STATE \textbf{Initialize:} $t \leftarrow 0$, $\bx^{(0)} \in \mathbb{R}^n$, $C^{(0)} \leftarrow \|\bx^{(0)}\|_2$, and $\epsilon^{(0)} > 0$.
\WHILE{$\|\bx^{(t)} - \bx^{(t-1)}\|_2 / \|\bx^{(t-1)}\|_2 > \tau_{out}$}
    \STATE Set $C \leftarrow C^{(t)}$, $\epsilon \leftarrow \epsilon^{(t)}$, $\mathbf{z}^{(0)} \leftarrow \bx^{(t)}$, and $k \leftarrow 0$.
    \REPEAT
        \STATE Compute the objective gradient $\nabla F_{\epsilon}(\mathbf{z}^{(k)}; C)$.
        \STATE Compute L-BFGS search direction $\mathbf{d}^{(k)} = -B^{(k)} \nabla F_{\epsilon}(\mathbf{z}^{(k)}; C)$.
        \STATE Determine step size $\alpha_k$ satisfying the Wolfe conditions via line search.
        \STATE $\mathbf{z}^{(k+1)} \leftarrow \mathbf{z}^{(k)} + \alpha_k \mathbf{d}^{(k)}$.
        \STATE $k \leftarrow k + 1$.
    \UNTIL{$\|\nabla F_{\epsilon}(\mathbf{z}^{(k)}; C)\|_2 \le \tau_{in}$}
    \STATE $\bx^{(t+1)} \leftarrow \mathbf{z}^{(k)}$.
    \STATE $C^{(t+1)} \leftarrow \|\bx^{(t+1)}\|_2$.
    \STATE $\epsilon^{(t+1)} \leftarrow \gamma \epsilon^{(t)}$.
    \STATE $t \leftarrow t + 1$.
\ENDWHILE
\ENSURE Reconstructed signal $\hat{\bx} = \bx^{(t)}$.
\end{algorithmic}
\end{algorithm}

 \section{Numerical Experiments}\label{sec:numerical}

In this section, we carry out numerical experiments to evaluate the empirical performance of the proposed ENZ-based entropy regularization framework against standard sparsity-promoting baselines. The experiments comprise sparse signal recovery from underdetermined measurements and gradient-domain image denoising. Unless stated otherwise, the ENZ model is implemented through the unconstrained penalized formulation minimizing the smoothed unnormalized entropic objective function $F_{\epsilon}(\bx; C)$ introduced in the previous section, executed via the continuation L-BFGS scheme detailed in \Cref{alg:enz_lbfgs}.

\subsection{Sparse Signal Recovery}

We first evaluate our framework on compressed sensing problems in the underdetermined regime ($m \ll n$). The measurement matrix $A \in \mathbb{R}^{m \times n}$ is generated from a correlated Gaussian ensemble where each row is sampled independently from $\mathcal{N}(\mathbf{0}, \Sigma)$ with covariance components $\Sigma_{ij} = (1 - r)\delta_{ij} + r$ for $r \in [0,1)$. In our benchmark simulations, we set $m=64$, $n=512$, and $r=0.1$. The ground-truth signal $\bx^\ast \in \mathbb{R}^n$ is constructed as a $k$-sparse vector whose support is selected uniformly at random, with its non-zero entries sampled independently from a standard normal distribution $\mathcal{N}(0,1)$. To evaluate algorithm robustness under significant amplitude heterogeneity, the non-zero coefficients are rescaled to satisfy a prescribed dynamic range $C_r = \log_{10}(\max |\bx^\ast| / \min |\bx^\ast|) = 3$. The observed linear measurements are generated via $\bb = A \bx^\ast + \bm{\varepsilon}$, where $\bm{\varepsilon}$ represents additive Gaussian noise scaled such that $\|\bm{\varepsilon}\|_2 = \eta\|A\bx^\ast\|_2$ for relative noise levels $\eta \in \{0.01, 0.02, 0.03\}$.

We compare the proposed unnormalized entropy regularization framework (setting $C=\|\bx\|_2$) against three classical baselines: the combinatorial $\ell_0$ norm, the convex $\ell_1$ relaxation, and the non-convex log-sum penalty. Each regularizer is solved using its respective standard numerical framework: Iterative Hard Thresholding (IHT) \cite{blumensath2009iterative} for $\ell_0$, the Iterative Shrinkage-Thresholding Algorithm (ISTA) \cite{beck2009fast} for $\ell_1$, and the Iteratively Reweighted $\ell_1$ (IRL1) algorithm \cite{wang2021nonconvex} for the log-sum penalty. Performance is measured using the relative reconstruction error $\mathrm{Err}(\hat{\bx}) = \|\hat{\bx} - \bx^\ast\|_2 / \|\bx^\ast\|_2$. Across $N=100$ independent Monte Carlo trials, a reconstruction is declared successful if $\mathrm{Err}(\hat{\bx}) \le \tau$ with a success threshold $\tau=0.05$, from which we compute the empirical success rate via the sample mean of the corresponding indicators $\mathbb{I}(\mathrm{Err}(\hat{\bx}) \le \tau)$.

\begin{figure}[ht]
    \centering
    \includegraphics[width=0.95\textwidth]{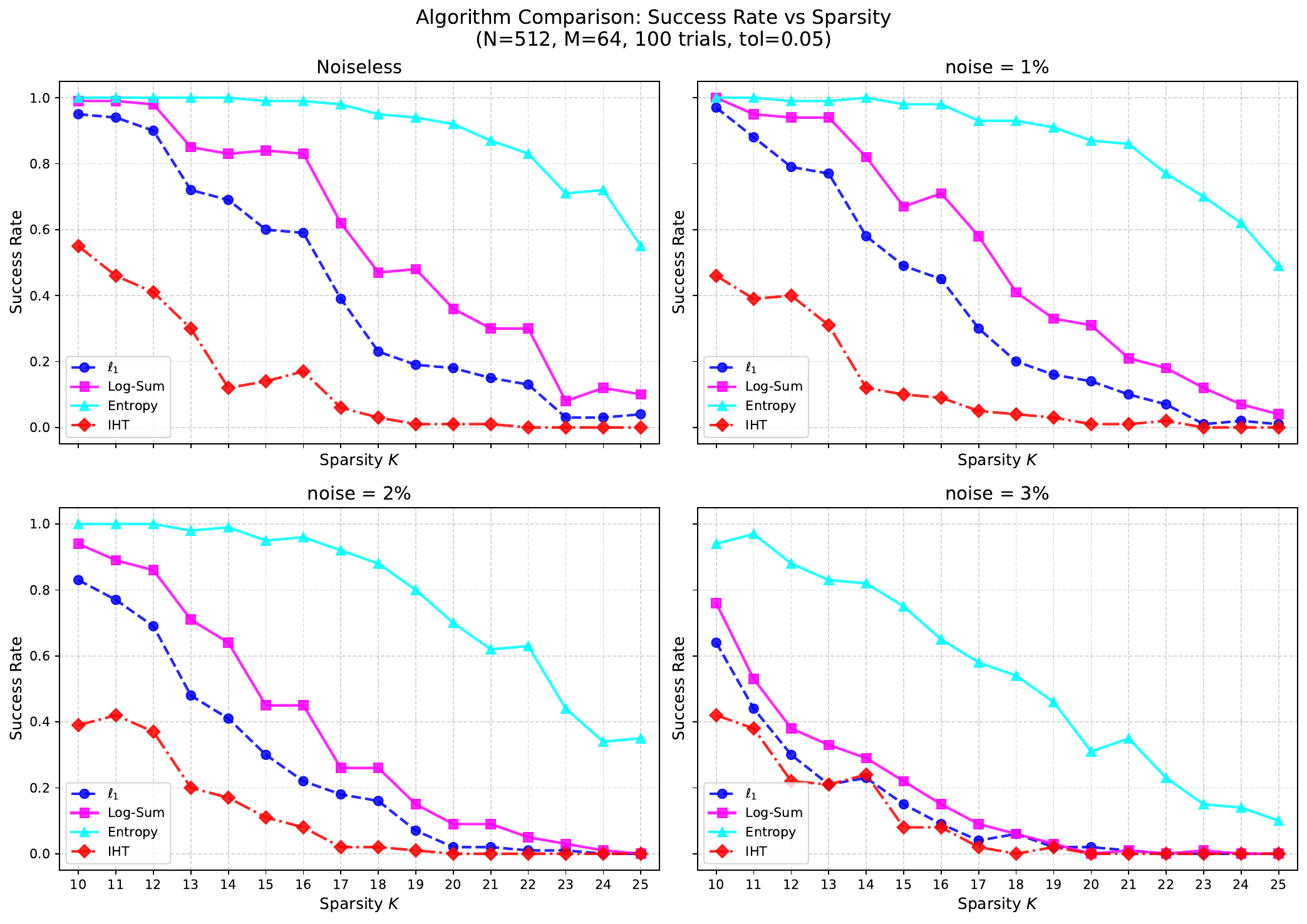}
    \caption{Empirical recovery success rate versus support sparsity level $k$ under correlated Gaussian sensing across different noise levels ($\eta\in\{0.01,0.02,0.03\}$). The proposed entropy regularizer maintains higher success rates across a wider sparsity regime and degrades more gracefully under noise than the $\ell_0$, $\ell_1$, and log-sum baselines.}
    \label{fig:success_rate}
\end{figure}

As summarized in \Cref{fig:success_rate}, the entropic regularization approach yields consistently higher success rates than the $\ell_1$ and log-sum baselines across all tested noise bounds. This confirms that evaluating the continuous energy distribution provides enhanced mathematical resilience when recovering sparse vectors characterized by large dynamic ranges.

\subsection{Image Denoising}

To assess the computational utility of the framework on structured physical systems, we formulate a gradient-domain variational image denoising problem. Let $\by \in \mathbb{R}^N$ denote the vectorized observed noisy image and $\bx \in \mathbb{R}^N$ represent the latent clean image ($N = H \times W$). Exploiting the piecewise-smooth nature of natural scenes described in \Cref{sec.concept}, we construct a discrete gradient operator $\mathbf{D} = [\mathbf{D}_x^\top, \mathbf{D}_y^\top]^\top \in \mathbb{R}^{2N \times N}$ by stacking the horizontal and vertical forward finite difference operators under periodic boundary conditions. The denoising task is cast as finding the optimal clean image minimizing a quadratic data fidelity term paired with our unnormalized entropic surrogate acting on the structural variations:
\begin{equation}
    \hat{\bx} = \arg\min_{\bx} \left( \frac{1}{2} \|\bx - \by\|_2^2 + \lambda F_{\epsilon}(\mathbf{D}\bx; C) \right),
\end{equation}
where $\lambda > 0$ balances noise removal and edge preservation. For our model, we fix the outer scaling constant to $C=1$, since the maximum possible intensity transition between adjacent pixels in a normalized image is exactly unity.

We evaluate the framework on standard test images ($128 \times 128$) corrupted by additive white Gaussian noise with standard deviation $\sigma=0.05$. The proposed approach is compared against standard Total Variation (TV) regularization \cite{rudin1992nonlinear}, which enforces an $\ell_1$ penalty on $\mathbf{D}\bx$, and the non-convex Log-Sum penalty \cite{Candes2008Enhancing}. Reconstruction quality is quantified via Peak Signal-to-Noise Ratio, $\text{PSNR}(\hat{\bx}, \bx) = 10 \log_{10} (\text{MAX}_I^2 / (\frac{1}{N}\|\hat{\bx} - \bx\|_2^2))$ with $\text{MAX}_I=1.0$, and the Structural Similarity Index Measure (SSIM) \cite{wang2004image}, which monitors luminance, contrast, and structural cohesion. To ensure an unbiased comparison, the regularizer hyperparameter $\lambda$ for all methods was optimized via a localized logarithmic grid search over $[10^{-4}, 10^{2}]$ to maximize the baseline PSNR metrics.

\begin{table}[H]
    \centering
    \caption{Quantitative comparison of image denoising algorithms (PSNR / SSIM) under white Gaussian noise ($\sigma=0.05$). The best values are boldfaced; the entropy regularizer consistently maximizes structural preservation across all benchmarks.}
    \label{tab:quantitative_results}
    \setlength{\tabcolsep}{8pt} 
    \renewcommand{\arraystretch}{1.2} 
    \begin{tabular}{lcccccc}
        \toprule
        \multirow{2}{*}{Method} & \multicolumn{2}{c}{Cameraman} & \multicolumn{2}{c}{Synthetic} & \multicolumn{2}{c}{Horse} \\
        \cmidrule(lr){2-3} \cmidrule(lr){4-5} \cmidrule(lr){6-7}
         & PSNR & SSIM & PSNR & SSIM & PSNR & SSIM \\
        \midrule
        TV       & 28.37 & 0.8198 & 32.62 & 0.2585 & 29.05 & 0.8140 \\
        Log-Sum  & 30.93 & 0.8375 & 34.01 & 0.2670 & \textbf{32.22} & 0.8249 \\
        \midrule
        Entropy  & \textbf{31.25} & \textbf{0.8760} & \textbf{34.14} & \textbf{0.2725} & 32.02 & \textbf{0.8295} \\
        \bottomrule
    \end{tabular}
\end{table}

\Cref{tab:quantitative_results} summarizes the quantitative results. The proposed entropy model achieves the highest PSNR on the \textit{Cameraman} and \textit{Synthetic} profiles, and delivers the highest SSIM across all three test sets. On the \textit{Horse} image, our framework attains the highest structural fidelity while maintaining a PSNR highly competitive with the Log-Sum benchmark (32.02 dB vs. 32.22 dB), demonstrating an excellent overall computational trade-off between smoothing and edge reconstruction.

\begin{figure*}[t]
    \centering
    \subfloat[Original]{\includegraphics[width=0.32\textwidth]{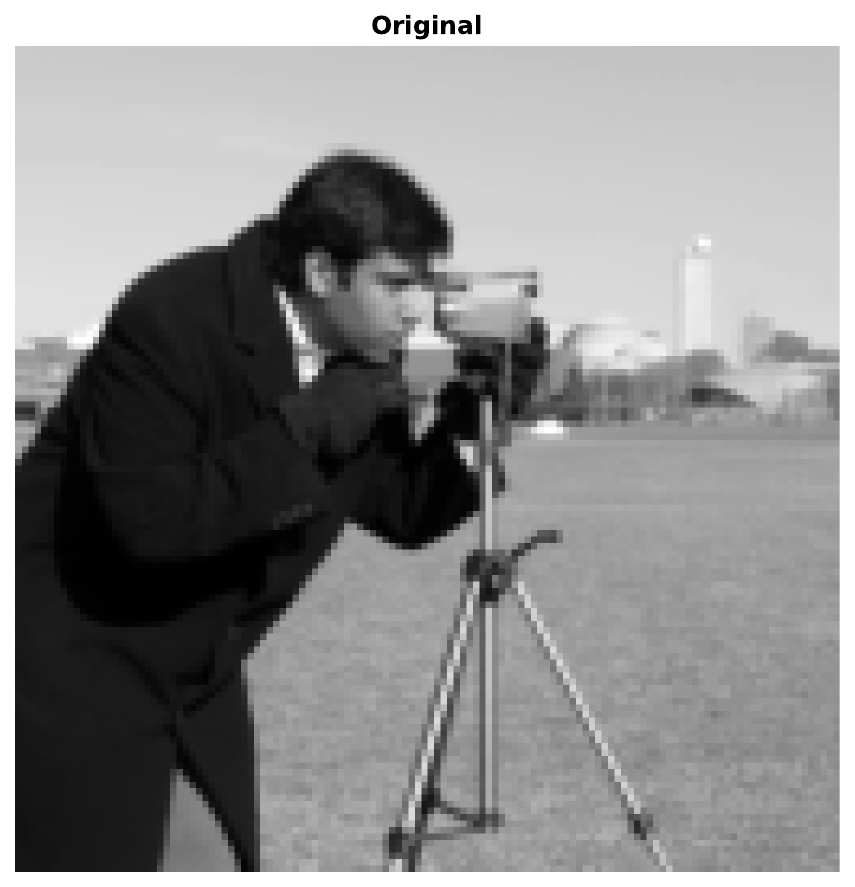}}\hfill
    \subfloat[Noisy]{\includegraphics[width=0.32\textwidth]{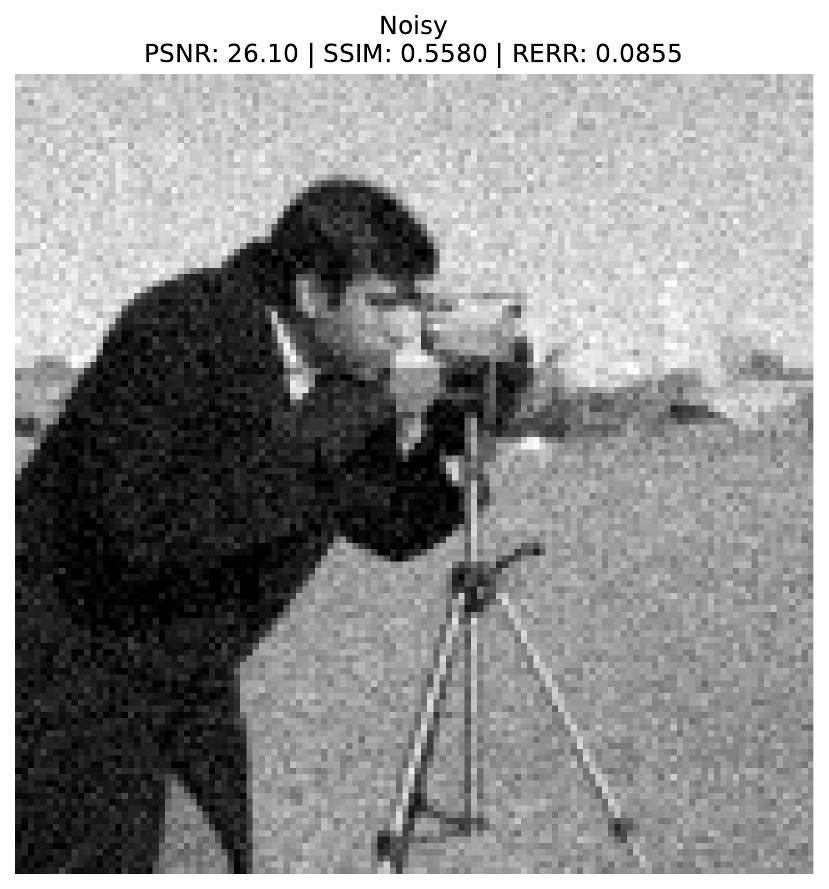}}\hfill
    \subfloat[TV]{\includegraphics[width=0.32\textwidth]{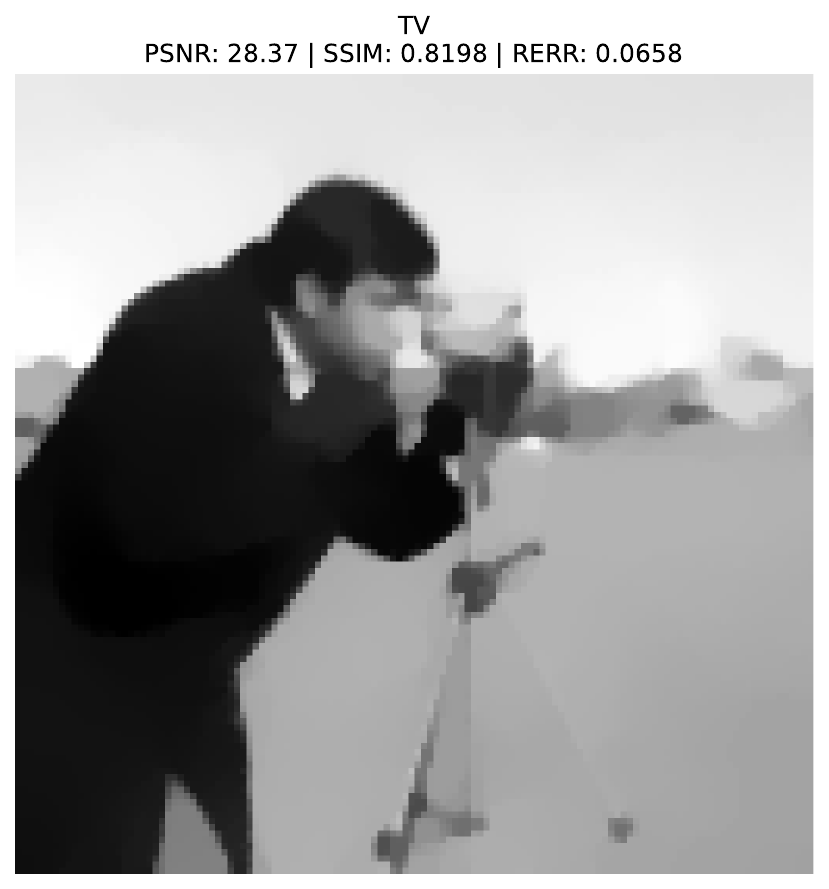}}\\[2pt]
    \makebox[\textwidth][c]{%
    \subfloat[Log-sum]{\includegraphics[width=0.32\textwidth]{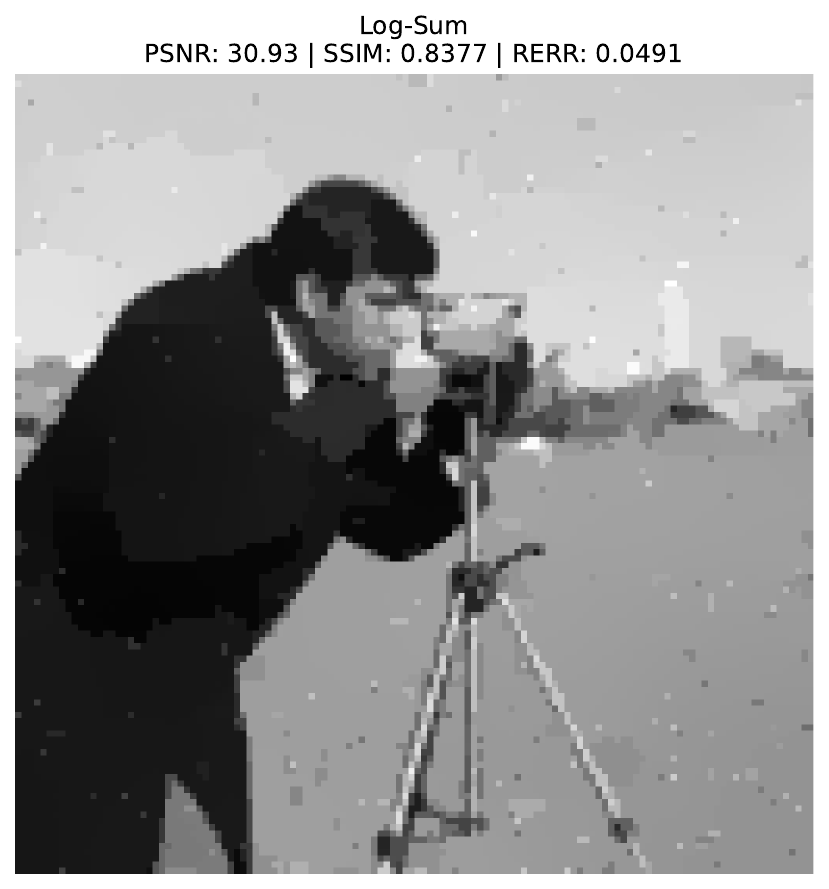}}\hspace{0.04\textwidth}
    \subfloat[Entropy]{\includegraphics[width=0.32\textwidth]{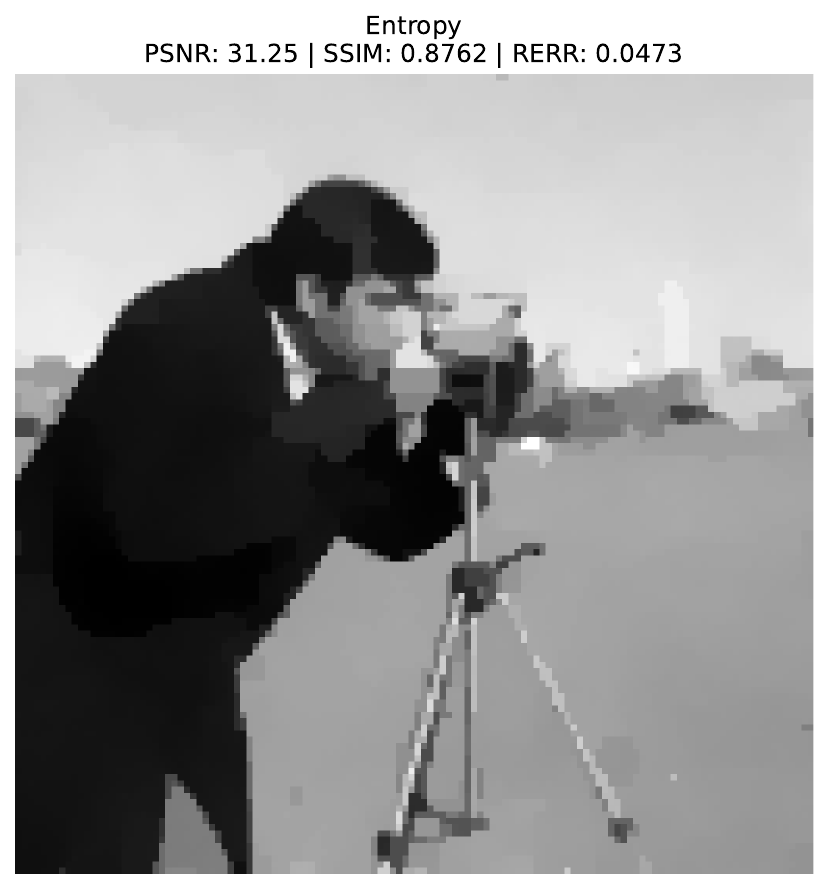}}}
    \caption{Visual denoising results on \textit{Cameraman} under white Gaussian noise ($\sigma=0.05$). Compared with the TV and Log-sum frameworks, the entropy regularizer suppresses residual noise artifacts without generating cartoon-like blurring on fine transitions.}
    \label{fig:cameraman}
\end{figure*}

\section{Conclusion}
\label{sec:conclusion}

In this paper, we present an information-theoretic framework for effective sparsity based on the effective number of nonzeros (ENZ). By moving from rigid binary counting to a continuous, relative magnitude-aware metric, the ENZ framework extends naturally to a parametric R\'enyi hierarchy and admits an exact informational decomposition that systematically discounts negligible long-tail noise. To bypass the computational coupling inherent to standard probability normalization, we introduce coordinate-separable unnormalized entropy surrogates and solve them within a double-loop framework leveraging iterative re-scaling and continuation L-BFGS. Numerical experiments in sparse signal recovery and image denoising demonstrate that this entropic approach improves reconstruction robustness and structural preservation compared to classical convex and non-convex penalties under noisy or correlated regimes.

{\noindent \bf Acknowledgment.}
This work was supported by the General Program of the National Natural Science Foundation of China (NSFC) under No. 12571326.

\bibliographystyle{siamplain}
\bibliography{reference}

\end{document}